\documentclass[10pt,twocolumn,letterpaper]{article}

\usepackage{cvpr}
\usepackage{times}
\usepackage{epsfig}
\usepackage{graphicx}
\usepackage{amsmath}
\usepackage{amssymb}

\usepackage{algorithm}
\usepackage{algorithmicx}
\usepackage{algpseudocode}
\usepackage{mathrsfs}
\usepackage{multirow}
\usepackage{amsthm}

\usepackage{bm}

\usepackage[title]{appendix}

\usepackage{caption2}
\usepackage{array}

\newtheorem{theorem}{Theorem}
\newtheorem{lemma}{Lemma}

\usepackage[pagebackref=true,breaklinks=true,letterpaper=true,colorlinks,bookmarks=false]{hyperref}

\cvprfinalcopy % *** Uncomment this line for the final submission

\def\cvprPaperID{4917} % *** Enter the CVPR Paper ID here

\ifcvprfinal\fi
\begin{document}

%%%%%%%%% TITLE
\title{Online high rank matrix completion}

\author{Jicong Fan, Madeleine Udell\\
Cornell University\\
Ithaca, NY 14853, USA\\
{\tt\small $\left\{\textup{jf577,udell}\right\}$@cornell.edu}
}

\maketitle
\thispagestyle{empty}

%%%%%%%%% ABSTRACT
\begin{abstract}
Recent advances in matrix completion enable data imputation
in full-rank matrices by exploiting low dimensional (nonlinear) latent structure.
In this paper, we develop a new model for high rank matrix completion (HRMC),
together with batch and online methods to fit the model and out-of-sample extension to complete new data.
The method works by (implicitly) mapping the data
into a high dimensional polynomial feature space using the kernel trick;
importantly, the data occupies a low dimensional subspace in this feature space,
even when the original data matrix is of full-rank.
We introduce an explicit parametrization of this low dimensional subspace,
and an online fitting procedure,
to reduce computational complexity compared to the state of the art.
The online method can also handle streaming or sequential data
and adapt to non-stationary latent structure.
% For example, the time complexity is reduced from $O(n^3)$ to $O(r^3)$,
% where $n$ is the number of data points, $r$ is the matrix rank in the feature space, and $r\ll n$.
We provide guidance on the sampling rate required these methods to succeed.
Experimental results on synthetic data and motion capture data
validate the performance of the proposed methods.
\end{abstract}

%%%%%%%%% BODY TEXT
\section{Introduction}

In the past ten years, low rank matrix completion (LRMC) has been widely studied \cite{CandesRecht2009,5466511,mazumder2010spectral,MC_ShattenP_AAAI125165,MC_IRRN,hardt2014understanding,MC_Universal2014,liu2016low,Fan2017290}. For instance, Cand\`{e}s and Recht \cite{CandesRecht2009} showed that any $n\times n$ incoherent matrices of rank $r$ can be exactly recovered from $Cn^{1.2}r\log n$ uniformly randomly sampled entries with high probability through solving a convex problem of nuclear norm minimization (NNM). However, LRMC cannot recover high rank or full-rank matrices,
even when the the data lies on a low dimensional (nonlinear) manifold. %latent dimension of the data is low.
To address this problem, recently a few researchers have developed new
high rank matrix completion (HRMC) methods \cite{eriksson2012high,li2016structured,yang2015sparse}
for data drawn from multiple subspaces \cite{SSC_PAMIN_2013,NIPS2016_6357,FAN2018SFMC}
or nonlinear models \cite{NLMC2016,pmlr-v70-ongie17a,FANNLMC}. These HRMC methods can outperform
LRMC methods for many real problems such as
subspace clustering with missing data,
motion data recovery \cite{NIPS2016_6357,pmlr-v70-ongie17a},
image inpainting, and classification \cite{NLMC2016,FANNLMC}.

All the aforementioned LRMC and HRMC methods are offline methods.
However, for many problems, we obtain one sample at a time and
would like to update the model as each new sample arrives using online optimization.
In addition, compared to offline methods,
online methods \cite{rendle2008online,mairal2009online,yun2015streaming}
often have lower space and time complexities and can adapt to changes in the latent data structure.
For these reasons, online matrix completion has recently gained increasing attention \cite{balzano2010online,dhanjal2014online,kawale2015efficient,lois2015online}.

%-------------------------------------------------------------------------
\section{Related work and our contribution}
\paragraph{Online matrix completion.} Sun and Luo \cite{sun2016guaranteed} and Jin et al. \cite{jin2016provable}
proposed to use stochastic gradient descent (SGD) to solve the low rank factorization (LRF) problem $\text{minimize} \sum_{(i,j)\in\Omega}
\left(\bm{X}_{ij}-\bm{U}_{i:}\bm{V}_{j:}^\top\right)^2$ with variables
$\bm{U}\in\mathbb{R}^{m\times r}$, $\bm{V}\in\mathbb{R}^{n\times r}$,
where $\bm{X}\in\mathbb{R}^{m\times n}$%, $r\geq\text{rank}(\bm{X})$,
and $\Omega$ denotes the locations of observed entries of $\bm{X}$.
Specifically, given an entry $\bm{X}_{ij}$, the $i$-th row of $\bm{U}$ and $j$-th row of $\bm{V}$ are updated by gradient descent.
Yun et al. \cite{yun2015streaming} studied the streaming or online matrix completion problem
when the columns of the matrix are presented sequentially.
The GROUSE method proposed in \cite{balzano2010online} used incremental gradient descent
on the Grassmannian manifold of subspaces to learn a low rank factorization from incomplete data online.
These online methods have a common limitation: they cannot recover high rank matrices.
Mairal et al. \cite{mairal2009online} also studied the online factorization problem
with the goal of learning a dictionary for sparse coding: $\underset{{\bm{D}\in\mathcal{C},\bm{\alpha}}}{\text{minimize}}\tfrac{1}{2}\Vert\bm{x}-\bm{D}\bm{\alpha}\Vert^2+\lambda\Vert \bm{\alpha}\Vert_1$.
A sparse factorization based matrix completion algorithm was proposed in \cite{FAN2018SFMC}.
It is possible to recover a high rank matrix online by combining ideas from \cite{mairal2009online} with \cite{FAN2018SFMC}.

\paragraph{High rank matrix completion.} Elhamifar \cite{NIPS2016_6357} proposed to use group-sparse constraint to complete high rank matrix consisting of data drawn from union of low-dimensional subspaces. Alameda-Pineda et al. \cite{NLMC2016} proposed a nonlinear matrix completion method for classification. The method performs matrix completion on a matrix consisting of (nonlinear) feature-label pairs, where the unknown labels are regarded as missing entries. The method is inapplicable to general matrix completion problems in which the locations of all missing entries are not necessarily in a single block. Ongie et al. \cite{pmlr-v70-ongie17a} assumed $\bm{X}$ is given by an algebraic variety and proposed a method called VMC to recover the missing entries of $\bm{X}$ through minimizing the rank of $\phi(\bm{X})$, where $\phi(\bm{X})$ is a feature matrix given by polynomial kernel. Fan and Chow \cite{FANNLMC} assumed the data are drawn from a nonlinear latent variable model and proposed a nonlinear matrix completion method (NLMC) that minimizes the rank of $\phi(\bm{X})$, where $\phi(\bm{X})$ is composed of high-dimensional nonlinear features induced by polynomial kernel or RBF kernel.

\paragraph{Challenges in HRMC.} First, existing HRMC methods lack strong theoretical guarantee
on the sample complexity required for recovery. For example, in VMC, the authors provide a lower bound of sampling rate ($\rho_0$, equation (6) of \cite{pmlr-v70-ongie17a}) only for low-order polynomial kernel and $\rho_0$ involved an unknown parameter $R$ owing to the algebraic variety assumption. In NLMC \cite{FANNLMC}, the authors only provided a coarse lower bound of sampling rate, i.e. $\rho>O(d/m)$, where $d$ is the dimension of latent variables. Second, existing HRMC methods are not scalable to large matrices. For example, VMC and NLMC require singular value decomposition on an $n\times n$ kernel matrix in every iteration. The method of \cite{NIPS2016_6357} is also not efficient because of the sparse optimization on an $n\times n$ coefficients matrix. Third, existing HRMC methods have no out-of-sample extensions, which means they cannot efficiently complete new data. Last but not least, existing HRMC methods are offline methods and cannot handle online data.

\paragraph{Contributions.}
In this paper, we aim to address these challenges.
We propose a novel high rank matrix completion method based on kernelized factorization (KFMC). KFMC is more efficient and accurate than state-of-the-art methods. Second, we propose an online version for KFMC, which can outperform online LRMC significantly. Third, we propose an out-of-sample extension for KFMC, which enables us to use the pre-learned high rank model to complete new data directly.
Finally, we analyze the sampling rate required for KFMC to succeed.

\section{Methodology}
\subsection{High rank matrices}\label{sec.s1}
We assume the columns of $\bm{X}\in\mathbb{R}^{m\times n}$ are given by
\begin{equation}\label{Eq.nl_0}
\bm{x}=\textbf{f}(\bm{s})=[f_1(\bm{s}),f_2(\bm{s}),\cdots,f_m(\bm{s})]^\top,
\end{equation}
where $\bm{s}\in\mathbb{R}^{d}$ ($d\ll m<n$) consists of uncorrelated variables and
each $f_i: \mathbb{R}^{d}\rightarrow \mathbb{R}$, $i=1,\ldots,m$, is a $p$-order polynomial with random coefficients.
For example, when $d=2$ and $p=2$, for $i=1,\ldots,m$, $x_i=\bm{c}_i^\top\bar{\bm{s}}$, where $\bm{c}_i\in\mathbb{R}^{6}$ and $\bar{\bm{s}}=[1,s_1,s_2,s_1^2,s_2^2,s_1s_2]^\top$.
Lemma \ref{theorem_1} shows that $\bm{X}$ is of high rank when $p$ is large.
\begin{lemma}\label{theorem_1}
Suppose the columns of $\bm{X}$ satisfy (\ref{Eq.nl_0}).
Then with probability 1, $\textup{rank}(\bm{X})=\min\lbrace m,n,\binom{d+p}{p}\rbrace$.
\end{lemma}
\begin{proof}
Expand the polynomial $f_i$ for each $i=1,\ldots,m$ to write $x_i=f_i(\bm{s})=\bm{c}_i^\top\bar{\bm{s}}$,
where $\bar{\bm{s}}=\{s_1^{\mu_1}\cdots s_d^{\mu_d}\}_{|\mu| \leq p}$
%[1,s_1,\ldots,s_d,s_1^2,\ldots,s_{d-1}s_d,\ldots,s_1^p,\ldots,s_{d-1}s_d^{p-1}]$
and $\bm{c}_i\in\mathbb{R}^{\binom{d+p}{p}}$.
Each column of $\bm{X}$ satisfies $\bm{x}=\bm{C}\bar{\bm{s}}$, where $\bm{C}=[\bm{c}_1 \cdots \bm{c}_m]\in\mathbb{R}^{m\times\binom{d+p}{p}}$.
The matrix $\bm{X}$ can be written as $\bm{X}=\bm{C}\bar{\bm{S}}$, where $\bar{\bm{S}}=(\bar{\bm{s}}_1^\top,\ldots,\bar{\bm{s}}_n^\top)\in\mathbb{R}^{\binom{d+p}{p}\times n}$.
The variables $\bm{s}$ are uncorrelated and the coefficients $\bm{c}$ are random,
so generically both $\bm{C}$ and $\bar{\bm{S}}$ are full rank.
Hence $\textup{rank}(\bm{X}) = \min\lbrace m,n, \binom{d+p}{p}\rbrace$.
\end{proof}

In this paper, our goal is to recover $\bm{X}$ from a few randomly sampled entries
we denote by $\lbrace\bm{M}_{ij}\rbrace_{(i,j)\in\Omega}$.
When $p$ is large, $\bm{X}$ is generically high rank and
cannot be recovered by conventional LRMC methods.

\vspace{2mm}
\noindent \textbf{Remark.} Throughout this paper,
we use the terms ``low rank'' or ``high rank" matrix to mean a matrix
whose rank is low or high \emph{relative} to its side length.
\vspace{2mm}

Let $\phi: \mathbb{R}^m\rightarrow \mathbb{R}^{\bar{m}}$
be a $q$-order polynomial feature map
$\phi(\bm{x})=\{x_1^{\mu_1} \cdots x_m^{\mu_m}\}_{|\mu|\leq q}$.
Here ${\bar{m}}=\binom{m+q}{q}$.
%  [1,\bar{\bm{x}}^{\lbrace 1\rbrace}\cdots,\bar{\bm{x}}^{\lbrace q\rbrace}]^\top$.
% Here $\bar{\bm{x}}^{\lbrace j\rbrace}$ denotes the set of $j$ combinations
% of the elements of $\bm{x}$ and ${\bar{m}}=\binom{m+q}{q}$.
Write $\phi(\bm{X})=[\phi(\bm{x}_1),\phi(\bm{x}_2),\cdots,\phi(\bm{x}_n)]$
and consider its rank:
\begin{theorem}\label{theorem_2}
Suppose the columns of $\bm{X}$ satisfy (\ref{Eq.nl_0}).
Then with probability 1, $\textup{rank}(\phi(\bm{X}))=\min\lbrace\bar{m},n,\binom{d+pq}{pq}\rbrace$.
\end{theorem}
\begin{proof}
Define the $pq$-order polynomial map $\psi(\bm{s}):=\phi(\bm{x})=\phi(\textbf{f}(\bm{s}))$.
Expanding as above, write the vector $\phi(\bm{x})=\bm{\Psi}\tilde{s}$
with $\bm{\Psi}\in\mathbb{R}^{\bar{m}\times \binom{d+pq}{pq}}$ and $\tilde{\bm{s}}\in\mathbb{R}^{\binom{d+pq}{pq}}$,
and write the matrix $\phi(\bm{X})=\bm{\Psi}\tilde{\bm{S}}$ with $\tilde{\bm{S}}=(\tilde{\bm{s}}_1^\top,\ldots,\tilde{\bm{s}}_n^\top)\in\mathbb{R}^{\binom{d+pq}{pq}\times n}$.
As above, $\bm{\Psi}$ and $\tilde{\bm{S}}$ are generically full rank, so $\textup{rank}(\phi(\bm{X}))=\min\lbrace\bar{m},n,\binom{d+pq}{pq}\rbrace$ with probability 1.
\end{proof}

While $\textup{rank}(\phi(\bm{X})) \geq \textup{rank}(\bm{X})$,
Theorem \ref{theorem_2} shows that $\phi(\bm{X})$ is generically low rank
when $d$ is small and $n$ is large.
For example, when $d=2$, $m=20$, $n=200$, $p=4$, and $q=2$,
generically $\tfrac{\text{rank}(\bm{X})}{\min\lbrace m,n\rbrace}=0.75$
while $\tfrac{\text{rank}(\phi(\bm{X}))}{\min\lbrace \bar{m},n\rbrace}=0.225$:
$\bm{X}$ is high rank but $\phi(\bm{X})$ is low rank.

\subsection{Kernelized factorization}\label{sec.kf}
To recover the missing entries of $\bm{X}$, we propose to solve
\begin{equation}\label{Eq.mc_00}
\begin{array}{ll}
\mathop{\text{minimize}}_{\bm{X},\bm{A},\bm{Z}} & \tfrac{1}{2}\Vert\phi(\bm{X})-\bm{A}\bm{Z}\Vert_F^2\\
\mbox{subject to} & \bm{X}_{ij}=\bm{M}_{ij},\ (i,j)\in\Omega,
\end{array}
\end{equation}
where $\bm{A}\in\mathbb{R}^{\bar{m}\times r}$, $\bm{Z}\in\mathbb{R}^{r\times n}$, and $r=\binom{d+pq}{pq}$.
The solution to (\ref{Eq.mc_00}) completes the entries of $\bm{X}$
using the natural low rank structure of $\phi(\bm{X})$.
Problem (\ref{Eq.mc_00}) implicitly defines an estimate for $f$,
$\hat{\textbf{f}}(\bm{S}):=\bm{X}\approx\phi^{-1}(\bm{A}\bm{Z})$.

For numerical stability, we regularize $\bm{A}$ and $\bm{Z}$ and solve
\begin{equation}\label{Eq.mc_0}
\begin{aligned}
\mathop{\text{minimize}}_{\bm{X},\bm{A},\bm{Z}}\ &\tfrac{1}{2}\Vert\phi(\bm{X})-\bm{A}\bm{Z}\Vert_F^2+\tfrac{\alpha}{2}\Vert\bm{A}\Vert_F^2+\tfrac{\beta}{2}\Vert\bm{Z}\Vert_F^2,\\
\mbox{subject to}\ &\bm{X}_{ij}=\bm{M}_{ij},\ (i,j)\in\Omega,
\end{aligned}
\end{equation}
where $\alpha$ and $\beta$ are regularization parameters, instead of (\ref{Eq.mc_00}).

It is possible to solve (\ref{Eq.mc_0}) directly but the computational cost is quite high if $m$ and $q$ are large.
The following lemma shows that there is no need to model $\bm{A}$ explicitly.
\begin{lemma}\label{lem_AD}
For any $\bm{X}$ generated by (\ref{Eq.nl_0}), there exist $\bm{D}\in\mathbb{R}^{m\times r}$ and $\bm{Z}\in\mathbb{R}^{r\times n}$ such that $\phi(\bm{X})=\phi(\bm{D})\bm{Z}$.
\end{lemma}
\begin{proof}
Suppose $\bm{D}\in\mathbb{R}^{m\times r}$ are also generated by (\ref{Eq.nl_0}) (e.g. any $r$ columns of $\bm{X}$),
so $\phi(\bm{D})$ and $\phi(\bm{X})$ share their column space
and (with probability 1) $\phi(\bm{D})$ is full rank. More precisely,
$\phi(\bm{D})=\bm{B}\bm{C}_{\scalebox{.6}{$\bm{D}$}}$ and $\phi(\bm{X})=\bm{B}\bm{C}_{\scalebox{.6}{$\bm{X}$}}$,
where $\bm{B}\in\mathbb{R}^{\bar{m}\times r}$ is a basis for the column space,
and both $\bm{C}_{\scalebox{.6}{$\bm{X}$}}\in\mathbb{R}^{r\times n}$
and $\bm{C}_{\scalebox{.6}{$\bm{D}$}}\in\mathbb{R}^{r\times r}$ are full rank.
Define $\bm{Z}=\bm{C}_{\scalebox{.6}{$\bm{D}$}}^{-1}\bm{C}_{\scalebox{.6}{$\bm{X}$}}$ and the result follows.
% It follows that $\phi(\bm{X})=\phi(\bm{D})\bm{C}_{\scalebox{.6}{$\bm{D}$}}^{-1}\bm{C}_{\scalebox{.6}{$\bm{X}$}}$. Letting $\bm{Z}=\bm{C}_{\scalebox{.6}{$\bm{D}$}}^{-1}\bm{C}_{\scalebox{.6}{$\bm{X}$}}$, we have $\phi(\bm{X})=\phi(\bm{D})\bm{Z}$ and $\phi(\bm{D})=\bm{A}$.
\end{proof}

\noindent Hence any solution of the following also solves (\ref{Eq.mc_0}):
% \begin{equation}\label{Eq.mc_1}
% \begin{aligned}
% \mathop{\text{minimize}}_{\bm{X},\bm{D},\bm{Z}}\ &\tfrac{1}{2}\Vert\phi(\bm{X})-\phi(\bm{D})\bm{Z}\Vert_F^2+\tfrac{\alpha}{2}\Vert\phi(\bm{D})\Vert_F^2+\tfrac{\beta}{2}\Vert\bm{Z}\Vert_F^2\\
% \text{subject to}\ &\bm{X}_{ij}=\bm{M}_{ij},\ (i,j)\in\Omega,
% \end{aligned}
% \end{equation}
\begin{align}
\mathop{\text{minimize}}_{\bm{X},\bm{D},\bm{Z}}\ &\tfrac{1}{2}\Vert\phi(\bm{X})-\phi(\bm{D})\bm{Z}\Vert_F^2+\tfrac{\alpha}{2}\Vert\phi(\bm{D})\Vert_F^2+\tfrac{\beta}{2}\Vert\bm{Z}\Vert_F^2 \nonumber\\
\text{subject to}\ &\bm{X}_{ij}=\bm{M}_{ij},\ (i,j)\in\Omega, \label{Eq.mc_1}
\end{align}
% \begin{array}{ll}
% \mathop{\text{minimize}}_{\bm{X},\bm{D},\bm{Z}} &\tfrac{1}{2}\Vert\phi(\bm{X})-\phi(\bm{D})\bm{Z}\Vert_F^2+\tfrac{\alpha}{2}\Vert\phi(\bm{D})\Vert_F^2+\tfrac{\beta}{2}\Vert\bm{Z}\Vert_F^2\\
% \text{subject to} &\bm{X}_{ij}=\bm{M}_{ij},\ (i,j)\in\Omega,
% \end{array}
where $\bm{D}\in\mathbb{R}^{m\times r}$ is much smaller than $\bm{A}\in\mathbb{R}^{\bar{m}\times r}$ of (\ref{Eq.mc_0}).
Use the trace function $\text{Tr}$ to
rewrite the objective in (\ref{Eq.mc_1}) as
\begin{align*}
%&\ell(\bm{Z},\bm{D},\bm{X})=
&\tfrac{1}{2}\text{Tr}\left(
\phi(\bm{X})^\top\phi(\bm{X})-2\phi(\bm{X})^\top\phi(\bm{D})\bm{Z}+\bm{Z}^\top\phi(\bm{D})^\top\phi(\bm{D})\bm{Z}\right) \\
&+\tfrac{\alpha}{2}\text{Tr}\left(\phi(\bm{D})^\top\phi(\bm{D})\right)+\tfrac{\beta}{2}\Vert\bm{Z}\Vert_F^2.
\end{align*}

Now we use the kernel trick to avoid explicitly computing the feature map $\phi$. Define $k(\bm{x},\bm{y}) := \phi(\bm{x})^\top\phi(\bm{y})=\langle\phi(\bm{x}),\phi(\bm{y})\rangle$,
so $\phi(\bm{X})^\top\phi(\bm{X})=\bm{K}_{\scalebox{.6}{$\bm{X}\bm{X}$}}$, $\phi(\bm{X})^\top\phi(\bm{D})=\bm{K}_{\scalebox{.6}{$\bm{X}\bm{D}$}}$, and $\phi(\bm{D})^\top\phi(\bm{D})=\bm{K}_{\scalebox{.6}{$\bm{D}\bm{D}$}}$, where $\bm{K}_{\scalebox{.6}{$\bm{X}\bm{X}$}}$, $\bm{K}_{\scalebox{.6}{$\bm{X}\bm{D}$}}$, and $\bm{K}_{\scalebox{.6}{$\bm{D}\bm{D}$}}$ are the corresponding kernel matrices. The most widely-used kernels are the polynomial kernel (Poly) and the radial basis function kernel (RBF)
\begin{equation}
\begin{aligned}
&\mbox{Poly}:\ k(\bm{x},\bm{y})=(\bm{x}^\top\bm{y}+c)^q\\
&\mbox{RBF}:\ k(\bm{x},\bm{y})=\exp\left(-\tfrac{1}{\sigma^2}\Vert \bm{x}-\bm{y}\Vert^2\right),
\end{aligned}
\end{equation}
with hyperparameters $c$, $q$, and $\sigma$. The (implicit) feature maps $\phi(\bm{x})$
of Poly and RBF are the $q$-order and infinite-order polynomial maps respectively.
Rewrite (\ref{Eq.mc_1}) to define kernelized factorizatiom matrix completion (KFMC)
\begin{equation}\label{Eq.mc_2} \tag{KFMC}
\begin{aligned}
\mathop{\text{minimize}}_{\bm{X},\bm{D},\bm{Z}}\ &\ell(\bm{Z},\bm{D},\bm{X})\\
\text{subject to}\ &\bm{X}_{ij}=\bm{M}_{ij},\ (i,j)\in\Omega
\end{aligned}
\end{equation}
where
$\ell(\bm{Z},\bm{D},\bm{X})=\tfrac{1}{2}\text{Tr}\left(
\bm{K}_{\scalebox{.6}{$\bm{X}\bm{X}$}} - 2\bm{K}_{\scalebox{.6}{$\bm{X}\bm{D}$}}\bm{Z}+\bm{Z}^\top\bm{K}_{\scalebox{.6}{$\bm{D}\bm{D}$}}\bm{Z}
\right)
+\tfrac{\alpha}{2}\text{Tr}(\bm{K}_{\scalebox{.6}{$\bm{D}\bm{D}$}})+\tfrac{\beta}{2}\Vert\bm{Z}\Vert_F^2.$
For the RBF kernel, $\text{Tr}(\bm{K}_{\scalebox{.6}{$\bm{D}\bm{D}$}})\equiv r$ is a constant
and can be dropped from the objective.

\subsection{Optimization for KFMC}
The optimization problem (\ref{Eq.mc_2}) is nonconvex and has three blocks of variables.
We propose using coordinate descent over these three blocks to find a stationary point.
%This section develops the resulting algorithm.
%shows how to solve the although the subproblems for $\bm{D}$ and $\bm{X}$ are difficult.

\paragraph{Update $\bm{Z}$.} To begin, complete entries of $\bm{X}$ arbitrarily and randomly initialize $\bm{D}$.
Define the $r \times r$ identity $\bm{I}_r$.
Fix $\bm{X}$ and $\bm{D}$ and update $\bm{Z}$ as
\begin{align}
\bm{Z}&\leftarrow\mathop{\text{arg min}}_{\bm{Z}}\ell(\bm{Z},\bm{D},\bm{X}) \nonumber\\
&=\mathop{\text{arg min}}_{\bm{Z}}-\text{Tr}(\bm{K}_{\scalebox{.6}{$\bm{X}\bm{D}$}}\bm{Z})+\tfrac{1}{2}\text{Tr}(\bm{Z}^\top\bm{K}_{\scalebox{.6}{$\bm{D}\bm{D}$}}\bm{Z})+\tfrac{\beta}{2}\Vert\bm{Z}\Vert_F^2 \nonumber\\
&=(\bm{K}_{\scalebox{.6}{$\bm{D}\bm{D}$}}+\beta \bm{I}_r)^{-1}\bm{K}_{\scalebox{.6}{$\bm{X}\bm{D}$}}^\top, \label{Eq.offline_Z_1}
\end{align}
% Then fix $\bm{Z}$ and $\bm{X}$ and update $\bm{D}$ as
% \begin{equation}\label{Eq.offline_D}
% \begin{aligned}
% \bm{D}&=\mathop{\text{arg min}}_{\bm{D}}\ell(\bm{Z},\bm{D},\bm{X})\\
% &=\mathop{\text{arg min}}_{\bm{D}}-\text{Tr}(\bm{K}_{\scalebox{.6}{$\bm{X}\bm{D}$}}\bm{Z})+\tfrac{1}{2}\text{Tr}(\bm{Z}^\top\bm{K}_{\scalebox{.6}{$\bm{D}\bm{D}$}}\bm{Z})+\tfrac{\alpha}{2}\text{Tr}(\bm{K}_{\scalebox{.6}{$\bm{D}\bm{D}$}}).
% \end{aligned}
% \end{equation}

\paragraph{Update $\bm{D}$.}
There is no closed form solution for the minimization of $\ell(\bm{Z},\bm{D},\bm{X})$ with respect to $\bm{D}$
due to the kernel matrices.
Instead, we propose the additive update
$\bm{D}\leftarrow\bm{D}-\bm{\Delta}_{\scalebox{.6}{$\bm{D}$}}$.
We compute $\bm{\Delta}_{\scalebox{.6}{$\bm{D}$}}$ using a relaxed Newton method,
described below for the Poly and RBF kernels.

For the polynomial kernel, rewrite the terms in the objective
in which $\bm{D}$ appears as
\begin{align}\label{Eq_objD}
\begin{aligned}
\ell(\bm{Z},\bm{D},\bm{X}):=&-\text{Tr}((\bm{W}_1\odot(\bm{X}^T\bm{D}+c))\bm{Z})\\
&+\tfrac{1}{2}\text{Tr}(\bm{Z}^T(\bm{W}_2\odot(\bm{D}^T\bm{D}+c))\bm{Z})\\
&+\tfrac{\alpha}{2}\text{Tr}(\bm{W}_2\odot(\bm{D}^T\bm{D}+c)).
\end{aligned}
\end{align}
defining $\bm{W}_1=\langle \bm{X}^\top\bm{D}+c\rangle^{q-1}$ and $\bm{W}_2=\langle\bm{D}^\top\bm{D}+c\rangle^{q-1}$,
where $\odot$ is elementwise multiplication and $\langle \cdot \rangle^u$ denotes the element-wise $u$-power.
Inspired by iteratively reweighted optimization,
fix $\bm{W}_1$ and $\bm{W}_2$ to approximate the gradient and Hessian
of $\ell(\bm{Z},\bm{D},\bm{X})$ with respect to $\bm{D}$ as %$\tfrac{\partial \ell}{\partial \bm{D}}$ as
\begin{align*}
\bm{g}_{\scalebox{.6}{$\bm{D}$}} &:= -\bm{X}(\bm{W}_1\odot\bm{Z}^\top)+\bm{D}((\bm{Z}\bm{Z}^\top + \alpha \bm{I}_r \odot\bm{W}_2)) \\
\bm{H}_{\scalebox{.6}{$\bm{D}$}} &:= \bm{Z}\bm{Z}^\top\odot\bm{W}_2+\alpha\bm{W}_2\odot\bm{I}_r.
\end{align*}
$\bm{H}_{\scalebox{.6}{$\bm{D}$}}$ is positive definite by the Schur product theorem.
Now choose $\tau>1$ for numerical stability and define the update
\begin{equation}\label{Eq.offline_D_delta_1}
\bm{\Delta}_{\scalebox{.6}{$\bm{D}$}}:=\tfrac{1}{\tau}\bm{g}_{\scalebox{.6}{$\bm{D}$}}\bm{H}_{\scalebox{.6}{$\bm{D}$}}^{-1}.
\end{equation}
The effectiveness of our update for $\bm{D}$ is guaranteed by the following lemma.
(The proof of the lemma and discussion about the role of $\tau$ are in the supplementary material.)
\begin{lemma}\label{Lem_offline_D_1}
The update (\ref{Eq.offline_D_delta_1}) is a relaxed Newton's method
and ensures sufficient decrease in the objective:
\[
\ell(\bm{Z},\bm{D}-\bm{\Delta}_{\scalebox{.6}{$\bm{D}$}},\bm{X})-\ell(\bm{Z},\bm{D},\bm{X})\leq-\tfrac{1}{2\tau}\textup{Tr}(\bm{g}_{\scalebox{.6}{$\bm{D}$}}\bm{H}_{\scalebox{.6}{$\bm{D}$}}^{-1}\bm{g}_{\scalebox{.6}{$\bm{D}$}}^\top).
\]
\end{lemma}

For the RBF kernel, the gradient is
\begin{equation}\label{Eq.offline_D_rbf0}
\nabla_{\scalebox{.6}{$\bm{D}$}}\ell=\tfrac{1}{\sigma^2}(\bm{X}\bm{Q}_1-\bm{D}\bm{\Gamma}_1)+\tfrac{2}{\sigma^2}(\bm{D}\bm{Q}_2-\bm{D}\bm{\Gamma}_2).
\end{equation}
(Throughout, we abuse notation to write $\ell$ for $\ell(\bm{Z},\bm{D},\bm{X})$.)
Here $\bm{Q}_1=-\bm{Z}^\top\odot\bm{K}_{\scalebox{.6}{$\bm{X}\bm{D}$}}$, $\bm{Q}_2=(0.5\bm{Z}\bm{Z}^\top+0.5\alpha\bm{I}_r)\odot\bm{K}_{\scalebox{.6}{$\bm{D}\bm{D}$}}$, $\bm{\Gamma}_1=\mbox{diag}(\bm{1}_n^\top\bm{Q}_1)$, and $\bm{\Gamma}_2=\mbox{diag}(\bm{1}_r^\top\bm{Q}_2)$, where $\bm{1}_n\in\mathbb{R}^n$ and $\bm{1}_r\in\mathbb{R}^r$ are composed of 1s. The following lemma (proved in the supplementary material) indicates that $\bm{X}\bm{Q}_1$ in (\ref{Eq.offline_D_rbf0}) is nearly a constant compared to $\bm{D}\bm{\Gamma}_1$, $\bm{D}\bm{Q}_2$, and $\bm{D}\bm{\Gamma}_2$, provided that $\sigma$ and $n$ are large enough:
\begin{lemma}\label{Lem_RBF_smallH1}
$\Vert\bm{X}(\bm{Z}^\top\odot\bm{K}_{\scalebox{.6}{$\bm{X}\bm{D}_1$}})-\bm{X}(\bm{Z}^\top\odot\bm{K}_{\scalebox{.6}{$\bm{X}\bm{D}_2$}})\Vert_F\leq\tfrac{c}{\sigma\sqrt{n}}\Vert \bm{X}\Vert_2\Vert\bm{D}_1-\bm{D}_2\Vert_F$, where $c$ is a small constant.
\end{lemma}
\noindent Therefore, we can compute an approximate Hessian neglecting $\bm{X}\bm{Q}_1$.
As in (\ref{Eq.offline_D_delta_1}), we define
\begin{equation}\label{Eq.offline_D_delta_2}
\bm{\Delta}_{\scalebox{.6}{$\bm{D}$}}:=\tfrac{1}{\tau}\nabla_{\scalebox{.6}{$\bm{D}$}}\ell(\tfrac{1}{\sigma^2}(2\bm{Q}_2-\bm{\Gamma}_1-2\bm{\Gamma}_2))^{-1}.
\end{equation}

\paragraph{Update $\bm{X}$.} Finally, fixing $\bm{Z}$ and $\bm{D}$, we wish to minimize (\ref{Eq.mc_2})
over $\bm{X}$, which again has no closed-form solution.
Again, we suggest updating $\bm{X}$ using a relaxed Newton method
$\bm{X}\leftarrow \bm{X}-\bm{\Delta}_{\scalebox{.6}{$\bm{X}$}}$.
For the polynomial kernel,
\begin{equation}
\begin{aligned}
\bm{g}_{\scalebox{.6}{$\bm{X}$}}&=\bm{X}(\bm{W}_3\odot\bm{I}_n)-\bm{D}(\bm{W}_4^\top\odot\bm{Z})\\
&=q\bm{X}\odot(\bm{1}_m\bm{w}^\top)-q\bm{D}(\bm{W}_4^\top\odot\bm{Z}),
\end{aligned}
\end{equation}
where $\bm{W}_3=\langle \bm{X}^\top\bm{X}+c\rangle^{q-1}$, $\bm{W}_4=\langle \bm{X}^\top\bm{D}+c\rangle^{q-1}$, $\bm{1}_m\in\mathbb{R}^m$ consists of 1s, and $\bm{w}\in\mathbb{R}^m$ consists of the diagonal entries of $\bm{W}_3$.
As above, we define %$\bm{\Delta}_{\scalebox{.6}{$\bm{X}$}}$ can be computed as
\begin{equation}\label{Eq.offline_X_delta_1}
\bm{\Delta}_{\scalebox{.6}{$\bm{X}$}}:=\tfrac{1}{\tau}\bm{g}_{\scalebox{.6}{$\bm{X}$}}\odot(\bm{1}_m\bm{w}^{-T}).
\end{equation}
When RBF kernel is used, we get
\begin{equation}
\nabla_{\scalebox{.6}{$\bm{X}$}}\ell=\tfrac{1}{\sigma^2}(\bm{D}\bm{Q}_3-\bm{X}\bm{\Gamma}_3)+\tfrac{2}{\sigma^2}(\bm{X}\bm{Q}_4-\bm{X}\bm{\Gamma}_4).
\end{equation}
Here $\bm{Q}_3=-\bm{Z}\odot\bm{K}_{\scalebox{.6}{$\bm{X}\bm{D}$}}^\top$, $\bm{Q}_4=0.5\bm{I}_n\odot\bm{K}_{\scalebox{.6}{$\bm{X}\bm{X}$}}$, $\bm{\Gamma}_3=\mbox{diag}(\bm{1}_r^\top\bm{Q}_3)$, and $\bm{\Gamma}_4=\mbox{diag}(\bm{1}_n^\top\bm{Q}_4)$.
As in (\ref{Eq.offline_D_delta_2}), define
\begin{equation}\label{Eq.offline_X_delta_2}
\bm{\Delta}_{\scalebox{.6}{$\bm{X}$}}:=\tfrac{1}{\tau}\nabla_{\scalebox{.6}{$\bm{X}$}}\ell(\tfrac{1}{\sigma^2}(2\bm{Q}_4-\bm{\Gamma}_3-2\bm{\Gamma}_4))^{-1}.
\end{equation}
Here the computational cost is not high in practice because the
matrix to be inverted is diagonal.

We can also use a momentum update to accelerate the convergence of $\bm{D}$ and $\bm{X}$:
\begin{equation}
\left\{
\begin{aligned}
&\widehat{\bm{\Delta}}_{\scalebox{.6}{$\bm{D}$}}\leftarrow\eta\widehat{\bm{\Delta}}_{\scalebox{.6}{$\bm{D}$}}+\bm{\Delta}_{\scalebox{.6}{$\bm{D}$}},\ \bm{D}\leftarrow\bm{D}-\widehat{\bm{\Delta}}_{\scalebox{.6}{$\bm{D}$}}\\
&\widehat{\bm{\Delta}}_{\scalebox{.6}{$\bm{X}$}}\leftarrow\eta\widehat{\bm{\Delta}}_{\scalebox{.6}{$\bm{X}$}}+\bm{\Delta}_{\scalebox{.6}{$\bm{X}$}},\ \bm{X}\leftarrow\bm{X}-\widehat{\bm{\Delta}}_{\scalebox{.6}{$\bm{X}$}}\\
\end{aligned}
\right.
\end{equation}
where $0<\eta<1$ is a constant.
The optimization method is summarized as Algorithm \ref{alg.OffLine_NLMC}.
The following lemma (with proof in the supplement) shows the method converges.
\begin{lemma}
For sufficiently small $\eta$, Algorithm \ref{alg.OffLine_NLMC} converges to a stationary point.
\end{lemma}
\renewcommand{\algorithmicrequire}{\textbf{Input:}}
\renewcommand{\algorithmicensure}{\textbf{Output:}}
\begin{algorithm}[h]
\caption{Offline KFMC}
\label{alg.OffLine_NLMC}
\begin{algorithmic}[1]
\Require
$\bm{M}$, $\Omega$, $r$, $k(\cdot,\cdot)$, $\alpha$, $\beta$, $t_{\text{max}}$, $\eta$
\State Initialize: $t=0$, $\bm{X}$, $\bm{D}\sim\mathcal{N}(0,1)$, $\widehat{\bm{\Delta}}_{\scalebox{.6}{$\bm{D}$}}=\bm{0}$, $\widehat{\bm{\Delta}}_{\scalebox{.6}{$\bm{X}$}}=\bm{0}$
\Repeat
\State $t\leftarrow t+1$
\State $\bm{Z}=(\bm{K}_{\scalebox{.6}{$\bm{D}\bm{D}$}}+\beta \bm{I}_r)^{-1}\bm{K}_{\scalebox{.6}{$\bm{X}\bm{D}$}}^\top$
\State Compute $\bm{\Delta}_{\scalebox{.6}{$\bm{D}$}}$ using (\ref{Eq.offline_D_delta_1}) or (\ref{Eq.offline_D_delta_2})
\State $\widehat{\bm{\Delta}}_{\scalebox{.6}{$\bm{D}$}}=\eta\widehat{\bm{\Delta}}_{\scalebox{.6}{$\bm{D}$}}+\bm{\Delta}_{\scalebox{.6}{$\bm{D}$}}$
\State  $\bm{D}\leftarrow\bm{D}-\widehat{\bm{\Delta}}_{\scalebox{.6}{$\bm{D}$}}$
\State Compute $\bm{\Delta}_{\scalebox{.6}{$\bm{X}$}}$ using (\ref{Eq.offline_X_delta_1}) or (\ref{Eq.offline_X_delta_2})
\State $\widehat{\bm{\Delta}}_{\scalebox{.6}{$\bm{X}$}}=\eta\widehat{\bm{\Delta}}_{\scalebox{.6}{$\bm{X}$}}+\bm{\Delta}_{\scalebox{.6}{$\bm{X}$}}$
\State $\bm{X}\leftarrow\bm{X}-\widehat{\bm{\Delta}}_{\scalebox{.6}{$\bm{X}$}}$ and $\bm{X}_{ij}=\bm{M}_{ij}\ \forall (i,j)\in\Omega$
\Until{converged or $t=t_{\text{max}}$}
\Ensure $\bm{X}$, $\bm{D}$
\end{algorithmic}
\end{algorithm}

\subsection{Online KFMC}\label{sec.s1}
Suppose we get an incomplete sample $\bm{x}$ at time $t$ and need to update the model of matrix completion timely or solve the optimization online. In (\ref{Eq.mc_1}), we can put the constraint into the objective function directly and get the following equivalent problem
\begin{equation}\label{Eq.mc_online_1}
\mathop{\text{minimize}}_{[\bm{X}]_{\bar{\Omega}},\bm{D},\bm{Z}}\ \sum_{j=1}^n\tfrac{1}{2}\Vert\phi(\bm{x}_j)-\phi(\bm{D})\bm{z}_j\Vert^2+\tfrac{\alpha}{2n}\Vert\phi(\bm{D})\Vert_F^2+\tfrac{\beta}{2}\Vert\bm{z}_j\Vert^2,
\end{equation}
where $[\bm{X}]_{\bar{\Omega}}$ denotes the unknown entries of $\bm{X}$.
Denote
\begin{equation}\label{Eq.mc_online_emperical}
\begin{aligned}
\ell([\bm{x}_j]_{\omega_j},\bm{D}):=\min\limits_{\bm{z}_j,[\bm{x}_j]_{\bar{\omega}_j}}&\tfrac{1}{2}\Vert\phi(\bm{x}_j)-\phi(\bm{D})\bm{z}_j\Vert^2\\
&+\tfrac{\alpha}{2n}\Vert\phi(\bm{D})\Vert_F^2+\tfrac{\beta}{2}\Vert\bm{z}_j\Vert^2,
\end{aligned}
\end{equation}
where $[\bm{x}_j]_{\omega_j}$ ($[\bm{x}_j]_{\bar{\omega}_j}$) denotes the observed (unknown) entries of $\bm{x}_j$ and $\omega_j$ ($\bar{\omega}_j$) denotes the corresponding locations. Then (\ref{Eq.mc_online_1}) minimizes the empirical cost function
\begin{equation}\label{Eq.mc_online_emperical_2}
g_n(\bm{D}):=\dfrac{1}{n}\sum_{j=1}^n\ell([\bm{x}_j]_{\omega_j},\bm{D}).
\end{equation}
The expected cost is
\begin{equation}\label{Eq.mc_online_3}
g(\bm{D}):=\mathbb{E}_{[\bm{x}]_{\omega}}\left[\ell([\bm{x}]_{\omega},\bm{D})\right]=\mathop\text{lim}_{n\rightarrow \infty}g_n(\bm{D}).
\end{equation}
To approximately minimize (\ref{Eq.mc_online_3}) online, we propose the following optimization for a given incomplete sample $\bm{x}$
\begin{equation}
\begin{aligned}
\mathop{\text{minimize}}_{[\bm{x}]_{\bar{\omega}},\bm{D},\bm{z}}\ &\hat{\ell}(\bm{z},[\bm{x}]_{\bar{\omega}},\bm{D}):= \tfrac{1}{2}\Vert\phi(\bm{x})-\phi(\bm{D})\bm{z}\Vert^2\\
&+\tfrac{\alpha}{2}\Vert\phi(\bm{D})\Vert_F^2+\tfrac{\beta}{2}\Vert\bm{z}\Vert^2.
\end{aligned}
\end{equation}
With randomly initialized $\bm{D}$, we first compute $\bm{z}$ and $[\bm{x}]_{\bar{\omega}}$ via alternately minimizing $\hat{\ell}(\bm{z},[\bm{x}]_{\bar{\omega}},\bm{D})$, which is equivalent to
\begin{equation}\label{Eq.mc_online_4}
\mathop{\text{minimize}}_{[\bm{x}]_{\bar{\omega}},\bm{z}}\ \tfrac{1}{2}k_{\scalebox{.6}{$\bm{x}\bm{x}$}}-\bm{k}_{\scalebox{.6}{$\bm{x}\bm{D}$}}\bm{z}+\tfrac{1}{2}\bm{z}^\top\bm{K}_{\scalebox{.6}{$\bm{D}\bm{D}$}}\bm{z}+\tfrac{\beta}{2}\Vert\bm{z}\Vert^2.
\end{equation}
Specifically, in each iteration, $\bm{z}$ is updated as
\begin{equation}\label{Eq.ol_zz}
\bm{z}=(\bm{K}_{\scalebox{.6}{$\bm{D}\bm{D}$}}+\beta\bm{I}_r)^{-1}\bm{k}_{\scalebox{.6}{$\bm{x}\bm{D}$}}^\top.
\end{equation}
We propose to update $[\bm{x}]_{\bar{\omega}}$ by Newton's method, i.e., $[\bm{x}]_{\bar{\omega}}\leftarrow[\bm{x}]_{\bar{\omega}}-[\bm{\Delta}_{\scalebox{.6}{$\bm{x}$}}]_{\bar{\omega}}$. When polynomial kernel is used, we obtain
\begin{equation}
\nabla_{\scalebox{.6}{$\bm{x}$}}\hat{\ell}=w_1\bm{x}-\bm{D}(\bm{w}_2^\top\odot\bm{z})
\end{equation}
where $w_1=\langle \bm{x}^\top\bm{x}+c\rangle^{q-1}$, $\bm{w}_2=\langle \bm{x}^\top\bm{D}+c\rangle^{q-1}$. Then
\begin{equation}\label{Eq.online_x_delta_1}
\bm{\Delta}_{\scalebox{.6}{$\bm{x}$}}=\tfrac{1}{\tau w_1}\nabla_{\scalebox{.6}{$\bm{x}$}}\hat{\ell}.
\end{equation}
When RBF kernel is used, we have
\begin{equation}
\nabla_{\scalebox{.6}{$\bm{x}$}}\hat{\ell}=\tfrac{1}{\sigma^2}(\bm{D}\bm{q}-\gamma\bm{x}),
\end{equation}
where $\bm{q}=-\bm{z}\odot\bm{k}_{\scalebox{.6}{$\bm{x}\bm{D}$}}^\top$ and $\gamma=\bm{1}_r^\top\bm{q}$. Then
\begin{equation}\label{Eq.online_x_delta_2}
\bm{\Delta}_{\scalebox{.6}{$\bm{x}$}}=\tfrac{\sigma^2}{\tau\gamma}\nabla_{\scalebox{.6}{$\bm{x}$}}\hat{\ell}.
\end{equation}
The derivations of (\ref{Eq.online_x_delta_1}) and (\ref{Eq.online_x_delta_2}) are similar to those of (\ref{Eq.offline_D_delta_1}) and (\ref{Eq.offline_D_delta_2}). Then we repeat (\ref{Eq.ol_zz})$-$(\ref{Eq.online_x_delta_2}) until converged.

After $\bm{z}$ and $[\bm{x}]_{\bar{\omega}}$ are computed, we compute $\bm{D}$ via minimizing $\hat{\ell}(\bm{z},[\bm{x}]_{\bar{\omega}},\bm{D})$, which is equivalent to
\begin{equation}
\mathop{\text{minimize}}_{\scalebox{.6}{$\quad \bm{D}$}}-\bm{k}_{\scalebox{.6}{$\bm{x}\bm{D}$}}\bm{z}+\tfrac{1}{2}\bm{z}^\top\bm{K}_{\scalebox{.6}{$\bm{D}\bm{D}$}}\bm{z}+\tfrac{\alpha}{2}\text{Tr}(\bm{K}_{\scalebox{.6}{$\bm{D}\bm{D}$}}).
\end{equation}
We propose to use SGD to update $\bm{D}$, i.e., $\bm{D}\leftarrow\bm{D}-\bm{\Delta}_{\scalebox{.6}{$\bm{D}$}}$. When polynomial kernel is used, we have
\begin{equation}
\nabla_{\scalebox{.6}{$\bm{D}$}}\hat{\ell}=-\bm{x}(\bm{w}_1\odot\bm{z}^\top)+\bm{D}(\bm{z}\bm{z}^\top\odot\bm{W}_2)+\alpha\bm{D}(\bm{W}_2\odot\bm{I}_r)),
\end{equation}
where $\bm{w}_1=\langle \bm{x}^\top\bm{D}+c\rangle^{q-1}$ and $\bm{W}_2=\langle\bm{D}^\top\bm{D}+c\rangle^{q-1}$.
Then we have
\begin{equation}\label{Eq.online_D_delta_1}
\bm{\Delta}_{\scalebox{.6}{$\bm{D}$}}=\tfrac{1}{\tau}\nabla_{\scalebox{.6}{$\bm{D}$}}\hat{\ell}/\Vert \bm{z}\bm{z}^\top\odot\bm{W}_2+\alpha\bm{W}_2\odot\bm{I}_r\Vert_2.
\end{equation}
Here we cannot use the method of (\ref{Eq.offline_D_delta_1}) because $\bm{z}\bm{z}^\top$ is not as stable as $\bm{Z}\bm{Z}^\top$. In addition, the following lemma (proved in the supplementary material) ensures the effectiveness of updating $\bm{D}$:
\begin{lemma}\label{Lem_online_D_1}
Updating $\bm{D}$ as $\bm{D}-\bm{\Delta}_{\scalebox{.6}{$\bm{D}$}}$ does not diverge and $\hat{\ell}(\bm{z},[\bm{x}]_{\bar{\omega}},\bm{D}-\bm{\Delta}_{\scalebox{.6}{$\bm{D}$}})-\hat{\ell}(\bm{z},[\bm{x}]_{\bar{\omega}},\bm{D})\leq -\tfrac{1}{2\tau\tau_0}\Vert\nabla_{\scalebox{.6}{$\bm{D}$}}\hat{\ell}\Vert_F^2$ provided that $\tau>1$, where $\tau_0=\Vert \bm{z}\bm{z}^\top\odot\bm{W}_2+\alpha\bm{W}_2\odot\bm{I}_r\Vert_2$.
\end{lemma}

When RBF kernel is used, the derivative is
\begin{equation}
\nabla_{\scalebox{.6}{$\bm{D}$}}\hat{\ell}=\tfrac{1}{\sigma^2}(\bm{x}\bm{Q}_1-\bm{D}\bm{\Gamma}_1)+\tfrac{2}{\sigma^2}(\bm{D}\bm{Q}_2-\bm{D}\bm{\Gamma}_2),\\
\end{equation}
where $\bm{Q}_1=-\bm{z}^\top\odot\bm{k}_{\scalebox{.6}{$\bm{X}\bm{D}$}}$, $\bm{Q}_2=(0.5\bm{z}\bm{z}^\top+0.5\alpha\bm{I}_r)\odot\bm{K}_{\scalebox{.6}{$\bm{D}\bm{D}$}}$, $\bm{\Gamma}_1=\mbox{diag}(\bm{Q}_1)$, and $\bm{\Gamma}_2=\mbox{diag}(\bm{1}_r^\top\bm{Q}_2)$. Similar to (\ref{Eq.online_D_delta_1}) and Lemma \ref{Lem_online_D_1}, we obtain
\begin{equation}\label{Eq.online_D_delta_2}
\bm{\Delta}_{\scalebox{.6}{$\bm{D}$}}=\tfrac{1}{\tau}\nabla_{\scalebox{.6}{$\bm{D}$}}\hat{\ell}/\Vert \tfrac{1}{\sigma^2}(2\bm{Q}_2-\bm{\Gamma}_1-2\bm{\Gamma}_2)\Vert_2.
\end{equation}

Similar to offline KFMC, we also use momentum to accelerate the optimization of online KFMC. The optimization steps are summarized in Algorithm \ref{alg.OnLine_NLMC}. The error of online matrix completion can be reduced with multi-pass optimization or increasing the number of samples. In Algorithm \ref{alg.OnLine_NLMC}, the sequence $\ell([\bm{x}_t]_{\omega_t},\bm{D})$ defined in (\ref{Eq.mc_online_emperical}) may not decrease continuously because the incomplete sample $\bm{x}_t$ can introduce high uncertainty. However, the sequence $g_t(\bm{D})$, the empirical cost function defined in (\ref{Eq.mc_online_emperical_2}), is convergent because for $j=1,\cdots,t$, $\ell([\bm{x}_j]_{\omega_j},\bm{D})$ is minimized through optimizing $[\bm{x}_j]_{\bar{\omega}_j}$, $\bm{z}_j$, and $\bm{D}$.

\renewcommand{\algorithmicrequire}{\textbf{Input:}}
\renewcommand{\algorithmicensure}{\textbf{Output:}}
\begin{algorithm}[h]
\caption{Online KFMC}
\label{alg.OnLine_NLMC}
\begin{algorithmic}[1]
\Require
Incomplete samples $\lbrace [\bm{x}_1]_{\omega_1},[\bm{x}_2]_{\omega_2},\cdots,[\bm{x}_t]_{\omega_t}\rbrace$, $r$, $k(\cdot,\cdot)$, $\alpha$, $\beta$, $n_{\text{iter}}$, $\eta$, $n_{\text{pass}}$
\State Initialize: $\bm{D}\sim\mathcal{N}(0,1)$, $\widehat{\bm{\Delta}}_{\scalebox{.6}{$\bm{D}$}}=\bm{0}$
\For{$u=1$ to $n_{\text{pass}}$}
\For{$j=1$ to $t$}
\State $l=0$, $\widehat{\bm{\Delta}}_{\scalebox{.6}{$\bm{X}$}}=\bm{0}$, $\bm{C}=(\bm{K}_{\scalebox{.6}{$\bm{D}\bm{D}$}}+\beta\bm{I}_r)^{-1}$
\Repeat
\State $l\leftarrow l+1$ and $\bm{z}_j=\bm{C}\bm{k}_{\scalebox{.6}{$\bm{X}\bm{D}$}}^\top$
\State Compute $\bm{\Delta}_{\scalebox{.6}{$\bm{x}$}}$ using (\ref{Eq.online_x_delta_1}) or (\ref{Eq.online_x_delta_2})
\State $\widehat{\bm{\Delta}}_{\scalebox{.6}{$\bm{x}$}}\leftarrow\eta\widehat{\bm{\Delta}}_{\scalebox{.6}{$\bm{x}$}}+\bm{\Delta}_{\scalebox{.6}{$\bm{x}$}}$
\State $[\bm{x}_j]_{\bar{\omega}_j}\leftarrow[\bm{x}_j]_{\bar{\omega}_j}-[\widehat{\bm{\Delta}}_{\scalebox{.6}{$\bm{x}$}}]_{\bar{\omega}_j}$
\Until {converged or $l=n_{\text{iter}}$}
\State Compute $\bm{\Delta}_{\scalebox{.6}{$\bm{D}$}}$ using (\ref{Eq.online_D_delta_1}) or (\ref{Eq.online_D_delta_2})
\State $\widehat{\bm{\Delta}}_{\scalebox{.6}{$\bm{D}$}}\leftarrow\eta\widehat{\bm{\Delta}}_{\scalebox{.6}{$\bm{D}$}}+\bm{\Delta}_{\scalebox{.6}{$\bm{D}$}}$ and $\bm{D}\leftarrow\bm{D}-\widehat{\bm{\Delta}}_{\scalebox{.6}{$\bm{D}$}}$
\EndFor
\EndFor
\Ensure $\bm{X}_t=[\bm{x}_1,\bm{x}_2,\cdots,\bm{x}_t]$, $\bm{D}$
\end{algorithmic}
\end{algorithm}

\subsection{Out-of-sample extension of KFMC}\label{sec.s1}
The matrix $\bm{D}$ obtained from offline matrix completion (\ref{alg.OffLine_NLMC}) or online matrix completion (\ref{alg.OnLine_NLMC}) can be used to recover the missing entries of new data
without updating the model.
We can also compute $\bm{D}$ from complete training data:
the corresponding algorithm is similar to Algorithms \ref{alg.OffLine_NLMC} and \ref{alg.OnLine_NLMC},
but does not require the $\bm{X}$ update.
We can complete a new (incomplete) sample $\bm{x}_{\text{new}}$ by solving
\begin{equation}\label{Eq.mc_ose_xnew}
\mathop{\text{minimize}}_{[\bm{x}_{\text{new}}]_{\bar{\omega}_{\text{new}}},\bm{z}_{\text{new}}}\ \tfrac{1}{2}\Vert\phi(\bm{x}_{\text{new}})-\phi(\bm{D})\bm{z}_{\text{new}}\Vert^2+\tfrac{\beta}{2}\Vert\bm{z}_{\text{new}}\Vert^2,\\
\end{equation}
where $[\bm{x}_{\text{new}}]_{\bar{\omega}_{\text{new}}}$ denotes unknown entries of $\bm{x}_{\text{new}}$.
This out-of-sample extension of KFMC is displayed as Algorithm \ref{alg.Outofsample}.

\renewcommand{\algorithmicrequire}{\textbf{Input:}}
\renewcommand{\algorithmicensure}{\textbf{Output:}}
\begin{algorithm}[h]
\caption{Out-of-sample extension for KFMC}
\label{alg.Outofsample}
\begin{algorithmic}[1]
\Require
$\bm{D}$ (computed from training data), $k(\cdot,\cdot)$, $\beta$, $n_{\text{iter}}$, $\eta$, new incomplete samples $\lbrace [\bm{x}_1]_{\omega_1},[\bm{x}_2]_{\omega_2},\cdots,[\bm{x}_t]_{\omega_t}\rbrace$
\State $\bm{C}=(\bm{K}_{\scalebox{.6}{$\bm{D}\bm{D}$}}+\beta\bm{I}_r)^{-1}$
\For{$j=1$ to $t$}
\State $l=0$, $\widehat{\bm{\Delta}}_{\scalebox{.6}{$\bm{x}$}}=\bm{0}$
\Repeat
\State $l\leftarrow l+1$ and $\bm{z}_j=\bm{C}\bm{k}_{\scalebox{.6}{$\bm{x}\bm{D}$}}^\top$
\State Compute $\bm{\Delta}_{\scalebox{.6}{$\bm{x}$}}$ using (\ref{Eq.online_x_delta_1}) or (\ref{Eq.online_x_delta_2})
\State $\widehat{\bm{\Delta}}_{\scalebox{.6}{$\bm{x}$}}\leftarrow\eta\widehat{\bm{\Delta}}_{\scalebox{.6}{$\bm{x}$}}+\bm{\Delta}_{\scalebox{.6}{$\bm{x}$}}$
\State $[\bm{x}_j]_{\bar{\omega}_j}\leftarrow[\bm{x}_j]_{\bar{\omega}_j}-[\widehat{\bm{\Delta}}_{\scalebox{.6}{$\bm{x}$}}]_{\bar{\omega}_j}$
\Until {converged or $l=n_{\text{iter}}$}
\EndFor
\Ensure $\bm{X}_{\text{new}}=[\bm{x}_1,\bm{x}_2,\cdots,\bm{x}_t]$
\end{algorithmic}
\end{algorithm}

\subsection{Complexity analysis}
Consider a high (even, full) rank matrix $\bm{X}\in\mathbb{R}^{m\times n}$ ($m\ll n$) given by (\ref{Eq.nl_0}).
In the methods VMC and NLMC, and our KFMC, the largest object stored is the kernel matrix $\bm{K}\in\mathbb{R}^{n\times n}$.
Hence their space complexities are all $O(n^2)$.
In VMC and NLMC, the major computational step is to compute $\bm{K}$ and its singular value decomposition at every iteration.
Hence their time complexities are $O(mn^2+n^3)$.
In our KFMC, the major computational steps are to form $\bm{K}$,
to invert the $r\times r$ matrix in (\ref{Eq.offline_Z_1}),
and to multiply an $m\times n$ and $n\times r$ matrix to compute the derivatives.
Hence the time complexity is $O(mn^2+r^3+rmn) = O(mn^2+rmn)$, since $n\gg r$.

Online KFMC does not store the kernel matrix $\bm{K}$.
Instead, the largest objects stored are $\bm{D}$ and $\bm{K}_{\scalebox{.6}{$\bm{D}\bm{D}$}}$.
Hence the space complexity is $O(mr+r^2)$.
The major computational step is to invert an $r\times r$ matrix (see Algorithm \ref{alg.OnLine_NLMC}).
Thus the time complexity is $O(r^3)$.
In the out-of-sample extension, the largest objects stored are $\bm{D}$ and $\bm{C}$
(see Algorithm \ref{alg.Outofsample}), so the space complexity is $O(mr+r^2)$.
For each online sample, we only need to multiply $m\times r$ matrices with vectors.
Hence the time complexity is just $O(mr)$.

This analysis are summarized in Table \ref{Tab.cc}.
We see that the space and time complexities of the proposed three approaches are much lower than those of VMC and NLMC.

\begin{table}
\begin{tabular}{lll}
\hline
	& Space complexity & Time complexity  \\
\hline
VMC & $O(n^2)$	&  $O(n^3+mn^2)$  \\
NLMC & $O(n^2)$	&  $O(n^3+mn^2)$  \\
KFMC &  $O(n^2)$ & $O(mn^2+rmn)$ \\
OL-KFMC & $O(mr+r^2)$ &	 $O(r^3)$	\\
OSE-KFMC & $O(mr+r^2)$	& $O(mr)$ \\
\hline
\end{tabular}
\caption{Time and space complexities {\footnotesize ($\bm{X}\in\mathbb{R}^{m\times n}$, $m\ll n$)}}\label{Tab.cc}
\end{table}

\subsection{Generalization for union of subspaces}
KFMC can also handle data drawn from a union of subspaces.
Suppose the columns of $\bm{X}\in\mathbb{R}^{m\times n}$ are given by
\begin{equation}\label{Eq.nl_uof_0}
\lbrace\bm{x}^{\lbrace k\rbrace}=\textbf{f}^{\lbrace k\rbrace}(\bm{s}^{\lbrace k\rbrace})\rbrace_{k=1}^u,
\end{equation}
where $\bm{s}^{\lbrace k\rbrace}\in\mathbb{R}^{d}$ ($d\ll m<n$) are random variables
and $\textbf{f}^{\lbrace k\rbrace}: \mathbb{R}^{d}\rightarrow \mathbb{R}^{m}$
are $p$-order polynomial functions for each $k=1,\ldots,u$. For convenience, we write
\begin{equation}
\bm{X}=[\bm{X}^{\lbrace 1\rbrace},\bm{X}^{\lbrace 2\rbrace},\cdots,\bm{X}^{\lbrace u\rbrace}],
\end{equation}
where the columns of each $\bm{X}^{\lbrace k\rbrace}$%\in\mathbb{R}^{m\times \tfrac{n}{u}}$
are in the range of $\textbf{f}^{\lbrace k\rbrace}$,
though we do not know which subspace each column of $\bm{X}$ is drawn from.
An argument similar to Lemma \ref{theorem_1} shows
\begin{equation}
\text{rank}(\bm{X})=\min\lbrace m,n,u\tbinom{d+p}{p}\rbrace
\end{equation}
with probability 1, so $\bm{X}$ is very likely to be of high rank or full rank when $u$ or $p$ is large.

We can generalize Theorem \ref{theorem_2} to show
$\text{rank}(\phi(\bm{X}))=\min\lbrace\bar{m},n,r\rbrace$ with probability 1,
where $\bar{m}=\binom{m+q}{q}$ and $r=u\binom{d+pq}{pq}$.
Hence when $d$ is small and $n$ is large, $\phi(\bm{X})$ is low rank,
so missing entries of $\bm{X}$ can still be recovered by the offline and online methods proposed in this paper.
In particular, for data drawn from a union of linear subspaces ($p=1$ and $u>1$),
generically $\text{rank}(\bm{X})=\min(m,n,ud)$ while $\text{rank}(\phi(\bm{X}))=u\binom{d+q}{q}$.

\subsection{On the sampling rate}
Suppose $\bm{X}$ is generated by (\ref{Eq.nl_uof_0}), and a proportion
$\rho_{\scalebox{.5}{KFMC}}$ of its entries are observed.
We provide some heuristics to help decide how many entries should be
observed for completion with the polynomial and RBF kernels.
Detailed calculations for (\ref{Eq.bound_poly}) and (\ref{Eq.poly2rbf}) are deferred to the supplement.

To complete $\phi(\bm{X})$ uniquely using a $q$-order polynomial kernel,
one rule of thumb is that the number of entries observed should be at least as
large as the number of degrees of freedom in the matrix $\phi(\bm{X})$
\cite{pmlr-v70-ongie17a}.
Here, $\phi(\bm{X})$ is a $\bar{m}\times n$ matrix with rank $r=u\binom{d+pq}{pq}$,
where $\bar{m}=\binom{m+q}{q}$.
We count the degrees of freedom of matrices with this rank to
argue sampling rate should satisfy
\begin{equation}\label{Eq.bound_poly}
\rho_{\scalebox{.5}{KFMC}} \geq \left(r/n+r/\bar{m}-r^2/n/\bar{m}\right)^{1/q}.
\end{equation}

Equation (\ref{Eq.bound_poly}) bounds the number of degrees of freedom
of $\phi(\bm{X})$ by considering its rank and size. But of course $\phi(\bm{X})$
is a deterministic function of $\bm{X}$, which has many fewer degrees of freedom.
Hence while (\ref{Eq.bound_poly}) provides a good rule of thumb,
we can still hope that lower sampling rates might produce sensible results.

For the RBF kernel, $q=\infty$, so the condition (\ref{Eq.bound_poly}) is vacuous.
However, the RBF kernel can be well approximated by a polynomial kernel and we have
\begin{equation}\label{Eq.poly2rbf}
\phi_i(\bm{x})=\hat{\phi}_i(\bm{x})+O(\sqrt{\tfrac{c^{q+1}}{(q+1)!}}),
\end{equation}
where $\hat{\phi}_i(\bm{x})$ denotes the $i$-th feature of $q$-order polynomial kernel and $\phi_i(\bm{x})$ denotes the $i$-th feature of the RBF kernel.
Hence exact recovery of $\hat{\phi}_i(\bm{x})$ implies
approximate recovery of $\phi_i(\bm{x})$ with error $O(\sqrt{\tfrac{c^{q+1}}{(q+1)!}})$.
This argument provides the intuition that the RBF kernel should recover the low-order ($\leq q$) features
of $\phi(\bm{x})$ with error $O(\sqrt{\tfrac{c^{q+1}}{(q+1)!}})$ provided that (\ref{Eq.bound_poly}) holds.
Of course, we can identify missing entries of $\bm{X}$ by considering the
first block of the completed matrix $\phi(\bm{X})$.

In experiments, we observe that the RBF kernel often works better than polynomial kernel.
We hypothesize two reasons for the effectiveness of the RBF kernel:
1) It captures the higher-order features in $\phi(\bm{x})$, which could be useful when $n$ is very large
2) It is easier to analyze and to optimize, speeding convergence.

Low rank matrix completion methods can only uniquely complete a matrix given
a sampling rate that satisfies
\begin{equation}\label{rho_LRMC}
\rho_{\scalebox{.5}{LRMC}}>\big((m+n)r_{\scalebox{.5}{$\bm{X}$}}-r_{\scalebox{.5}{$\bm{X}$}}^2\big)/(mn),
\end{equation}
where $r_{\scalebox{.5}{$\bm{X}$}}=\min\lbrace m,n,u\tbinom{d+p}{p}\rbrace$.
This bound can be vacuous (larger than 1) if $u$ or $p$ are large.
In contrast, $\rho_{\scalebox{.5}{KFMC}}$ given by (\ref{Eq.bound_poly}) can still
be smaller than 1 in the same regime, provided that $n$ is large enough.
For example, when $m=20$, $d=2$, $p=2$, and $u=3$, we have $\rho_{\scalebox{.5}{LRMC}}>0.91$.
Let $q=2$ and $n=300$, we have $\rho_{\scalebox{.5}{KFMC}}>0.56$.
If $p=1$ and $u=10$, we have $\rho_{\scalebox{.5}{LRMC}}>1$ and $\rho_{\scalebox{.5}{KFMC}}>0.64$.
This calculation provides further intuition for how our methods can recover high rank matrices
while classical low rank matrix completion methods fail.

\subsection{Analytic functions and smooth functions}
Hitherto, we have assumed that $\textbf{f}$ is a finite order polynomial function.
However, our methos also work when $\textbf{f}$ is an analytic or smooth function.
Analytic functions are well approximated by polynomials.
Furthermore, smooth functions can be well approximated by polynomial functions at least on intervals.
Hence for a smooth function $\textbf{h}: \mathbb{R}^{d}\rightarrow \mathbb{R}^m$,
we consider the generative model
\begin{equation}\label{Eq.nl_smooth}
\bm{x}=\textbf{h}(\bm{s})=\textbf{f}(\bm{s})+\bm{e}
\end{equation}
where $\textbf{f}$ is a $p$-order Taylor expansion of $\textbf{h}$
and $\bm{e}\in\mathbb{R}^m$ denotes the residual,
which scales as $\bm{e}\sim O(\tfrac{c}{(p+1)!})$ where $c$ is the
$p+1$th derivative of $\textbf{h}$.

We see that the error $\bm{e}$ from our polynomial model decreases as $p$ increases.
To fit a model with larger $p$, the bound (\ref{Eq.bound_poly}) suggests we need more samples $n$.
We conjecture that for any smooth $\textbf{h}$,
it is possible to recover the missing entries
with arbitrarily low error provided $n$ is sufficiently large.

\section{Experiments}
\subsection{Synthetic data}\label{sec.syn}
We generate the columns of $\bm{X}$ by $\bm{x}=\textbf{f}(\bm{s})$ where $\bm{s}\sim\mathcal{U}(0,1)$ and $\textbf{f}:\mathbb{R}^3\rightarrow\mathbb{R}^{30}$ is a $p$-order polynomial mapping. The model can be reformulated as $\bm{x}=\bm{P}\bm{z}$, where $\bm{P}\in\mathbb{R}^{30\times{(\binom{3+p}{p}-1)}}$, $\bm{P}\sim\mathcal{N}(0,1)$, and $\bm{z}$ consists of the polynomial features of $\bm{s}$. Consider the following cases:
\vspace{-3mm}
\begin{itemize}
\setlength{\itemsep}{0pt}
\setlength{\parsep}{0pt}
\setlength{\parskip}{0pt}
\item \textit{Single nonlinear subspace} Let $p=3$, generate one $\bm{P}$ and 100 $\bm{s}$. Then the rank of $\bm{X}\in\mathbb{R}^{30\times 100}$ is 19.
\item \textit{Union of nonlinear subspaces} Let $p=3$, generate three $\bm{P}$ and for each $\bm{P}$ generate 100 $\bm{s}$. Then the rank of $\bm{X}\in\mathbb{R}^{30\times 300}$ is 30.
\item \textit{Union of linear subspaces} Let $p=1$, generate ten $\bm{P}$ and for each $\bm{P}$ generate 100 $\bm{s}$. Then the rank of $\bm{X}\in\mathbb{R}^{30\times 1000}$ is 30.
\end{itemize}
\vspace{-1mm}
We randomly remove some portions of the entries of the matrices and use matrix completion to recover the missing entries. The performances are evaluated by the relative error defined as $RE=\Vert \widehat{\bm{X}}-\bm{X}\Vert_F/\Vert \bm{X}\Vert_F$ \cite{NIPS2016_6357}, where $\widehat{\bm{X}}$ denotes the recovered matrix. As shown in Figure \ref{Fig_syn_off}, the recovery errors of LRMC methods, i.e. LRF \cite{sun2016guaranteed} and NNM \cite{CandesRecht2009}, are considerably high. In contrast, HRMC methods especially our KFMC have significantly lower recovery errors. In addition, our KFMC(Poly) and KFMC(RBF) are much more efficient than VMC \cite{pmlr-v70-ongie17a} and NLMC \cite{FANNLMC}, in which randomized SVD \cite{randomsvd} has been performed.

Figure \ref{Fig_syn_ol} shows the results of online matrix completion, in which OL-DLSR (dictionary learning and sparse representation) is an online matrix completion method we modified from \cite{mairal2009online} and \cite{FAN2018SFMC} and detailed in our supplementary material. We see that our method OL-KFMC outperformed other methods significantly. Figures \ref{Fig_syn_ose} shows the results of out-of-sample extension (OSE) of HRMC, in which our OSE-KFMC outperformed other methods. More details about the experiment/parameter settings and analysis are in the supplementary material.

\begin{figure}[h!]
\centering
\includegraphics[width=8cm]{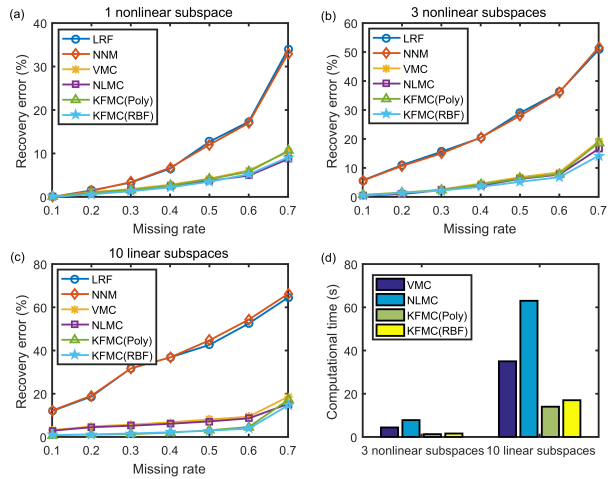}
\vspace{-2mm}
\caption{Offline matrix completion on synthetic data}
\label{Fig_syn_off}
\vspace{-1mm}
\end{figure}

\begin{figure}[h!]
\centering
\includegraphics[width=8cm]{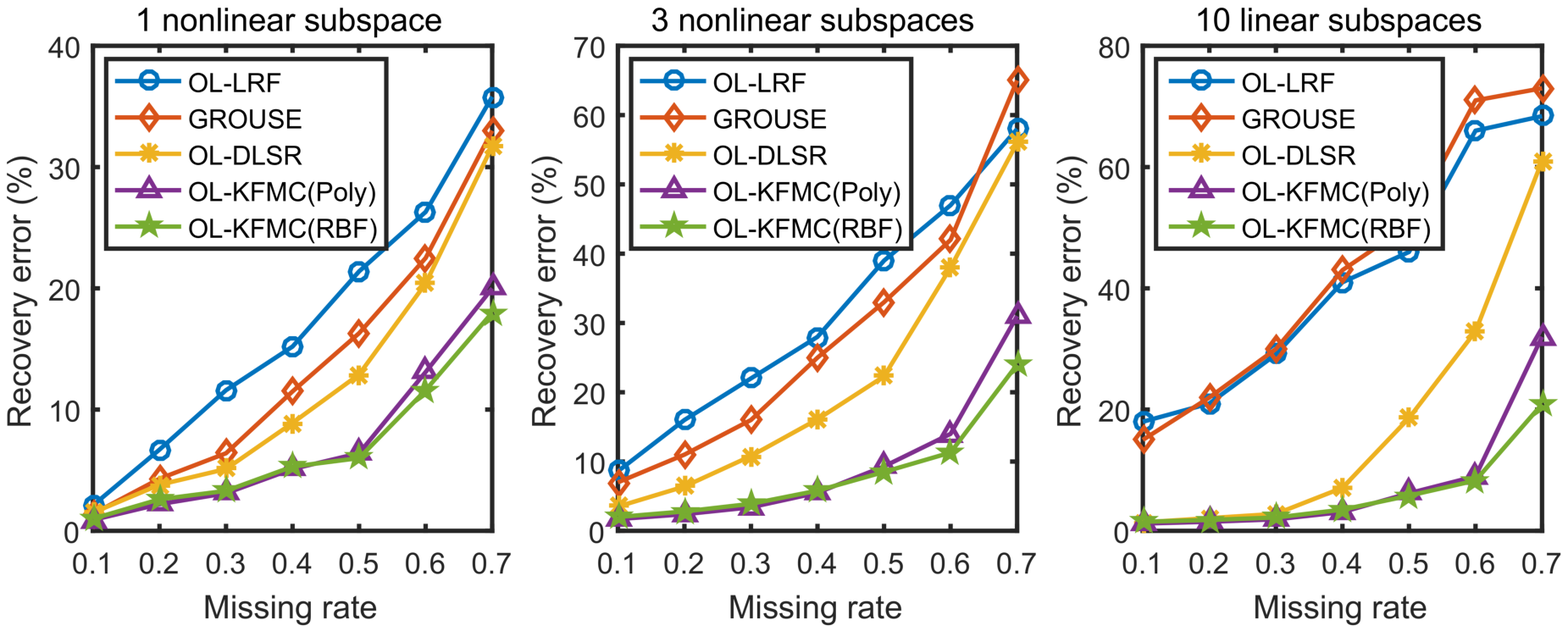}
\vspace{-2mm}
\caption{Online matrix completion on synthetic data}
\label{Fig_syn_ol}
\vspace{-1mm}
\end{figure}

\begin{figure}[h!]
\centering
\includegraphics[width=8cm]{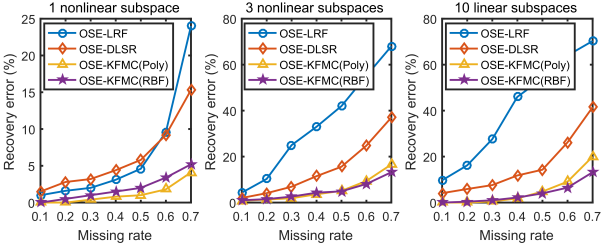}
\vspace{-2mm}
\caption{Out-of-sample extension of matrix completion on synthetic data}
\label{Fig_syn_ose}
\vspace{-1mm}
\end{figure}

\subsection{Hopkins155 data}
Similar to \cite{pmlr-v70-ongie17a}, we consider the problem of subspace clustering on incomplete data, in which the missing data of Hopkins155 \cite{4269999} were recovered by matrix completion and then SSC (sparse subspace clustering \cite{SSC_PAMIN_2013}) was performed. We consider two downsampled video sequences, \textit{1R2RC} and \textit{1R2TCR}, each of which consists of 6 frames. The average clustering errors \cite{SSC_PAMIN_2013} of 10 repeated trials are reported in Figure \ref{Fig_hopkins}. Our method KFMC with RBF kernel is more accurate and efficient than VMC and NLMC.

\begin{figure}[h!]
\centering
\includegraphics[width=8cm]{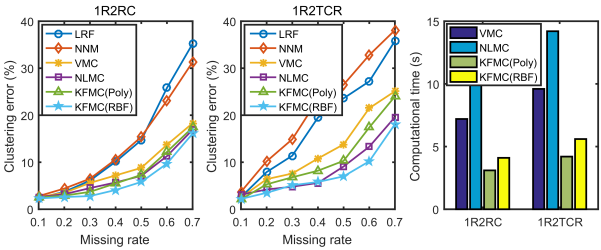}
\vspace{-2mm}
\caption{Subspace clustering on incomplete data}
\label{Fig_hopkins}
\vspace{-2mm}
\end{figure}

\subsection{CMU motion capture data}
We use matrix completion to recover the missing data of time-series trajectories of human motions (e.g. running and jumping). Similar to \cite{NIPS2016_6357,pmlr-v70-ongie17a}, we use the trial $\#6$ of subject $\#56$ of the CMU motion capture dataset, which forms a high rank matrix \cite{NIPS2016_6357}. We consider two cases of incomplete data, randomly missing and continuously missing. More details about the experimental settings are in the supplementary material. The average results of 10 repeated trials are reported in Figure \ref{Fig_cmu}. We see that HRMC methods outperformed LRMC methods while online methods outperformed offline methods. One reason is that the structure of the data changes with time (corresponding to different motions) and online methods can adapt to the changes. Comparing Figure \ref{Fig_cmu} with the Figure 4 of \cite{NIPS2016_6357}, we find that VMC, NLMC, and our KFMC outperformed the method proposed in \cite{NIPS2016_6357}. In addition, our OL-KFMC especially with RBF kernel is the most accurate one. Regarding the computational cost, there is no doubt that the linear methods including LRF, NNM, GROUSE, and DLSR are faster than other methods. Hence we only show the computational cost of the nonlinear methods in Table \ref{Tab_time} for comparison (randomized SVD \cite{randomsvd} has been performed in VMC and NLMC). Our KFMC is faster than VMC and NLMC while our OL-KFMC is at least 10 times faster than all methods.

\begin{figure}[h!]
\centering
\includegraphics[width=8cm]{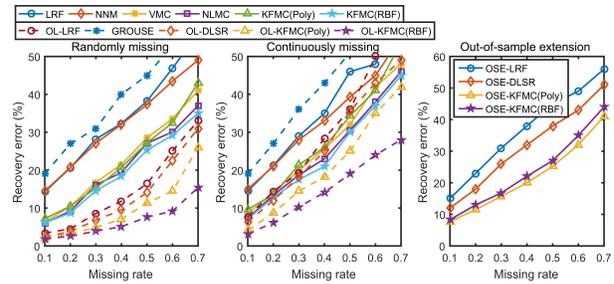}
\vspace{-3mm}
\caption{CMU motion capture data recovery}
\label{Fig_cmu}
\end{figure}

\vspace{-2mm}

\begin{table}[h!]
\centering
\begin{tabular}{p{2.2cm}p{2.42cm}p{2.42cm}}
\hline
{\footnotesize VMC} \qquad 370 & {\footnotesize NLMC } \qquad \ \ \ 610 & {\footnotesize KFMC}{\tiny (Poly)}  \quad \ \ \ 170 \\ \hline
{\footnotesize KFMC}{\tiny (RBF)} \ 190 &{\footnotesize OL-KFMC}{\tiny (Poly)} \ 16 &{\footnotesize OL-KFMC}{\tiny (RBF)} \ 19\\
\hline
\end{tabular}
\caption{Time cost (second) on CMU motion capture data}
\label{Tab_time}
\end{table}
\vspace{-3mm}

\section{Conclusion}
In this paper, we proposed kernelized factorization matrix completion (KFMC)\footnote{The codes are available at https://github.com/jicongfan/Online-high-rank-matrix-completion},
a new method for high rank matrix completion,
together with an online version and an out-of-sample extension,
which outperform state-of-the-art methods.
Our numerics demonstrate the success of the method for motion data recovery.
We believe our methods will also be useful for transductive learning (classification),
vehicle/robot/chemistry sensor signal denoising, recommender systems, and biomedical data recovery.
%-------------------------------------------------------------------------
\section*{Acknowledgement}
This work was supported in part by DARPA Award FA8750-17-2-0101.
%\clearpage

{\small
\bibliographystyle{ieee}
\bibliography{Ref_OLHRMC}
}

\appendix
\section*{Appendix}
\section{Proof of some lemmas}
\subsection{Proof for Lemma 3}
\noindent \textbf{Lemma 3.}
\textit{The update (8) is a relaxed Newton's method
and ensures sufficient decrease in the objective:
\[
\ell(\bm{Z},\bm{D}-\bm{\Delta}_{\scalebox{.6}{$\bm{D}$}},\bm{X})-\ell(\bm{Z},\bm{D},\bm{X})\leq-\tfrac{1}{2\tau}\textup{Tr}(\bm{g}_{\scalebox{.6}{$\bm{D}$}}\bm{H}_{\scalebox{.6}{$\bm{D}$}}^{-1}\bm{g}_{\scalebox{.6}{$\bm{D}$}}^\top).
\]}

\begin{proof}
With polynomial kernel, the objective function in terms of $\bm{D}$ is
\begin{equation}\label{Eq_objD}
\begin{aligned}
\ell(\bm{Z},\bm{D},\bm{X})=&-\text{Tr}((\bm{W}_1\odot(\bm{X}^T\bm{D}+c))\bm{Z})\\
&+\tfrac{1}{2}\text{Tr}(\bm{Z}^T(\bm{W}_2\odot(\bm{D}^T\bm{D}+c))\bm{Z})\\
&+\tfrac{\alpha}{2}\text{Tr}(\bm{W}_2\odot(\bm{D}^T\bm{D}+c)),
\end{aligned}
\end{equation} 
in which for simplicity we have omitted the terms not related to $\bm{D}$. In (\ref{Eq_objD}), $\bm{W}_1=\langle \bm{X}^T\bm{D}+c\rangle^{q-1}$, $\bm{W}_2=\langle\bm{D}^T\bm{D}+c\rangle^{q-1}$, and $\langle \cdot \rangle^u$ denotes the element-wise $u$-power of a vector or matrix. Using the idea of iteratively reweighted optimization, we fix $\bm{W}_1$ and $\bm{W}_2$, and get the derivative as
\begin{equation}
\bm{g}_{\scalebox{.6}{$\bm{D}$}}:=-\bm{X}(\bm{W}_1\odot\bm{Z}^T)+\bm{D}(\bm{Z}\bm{Z}^T\odot\bm{W}_2)+\alpha\bm{D}(\bm{W}_2\odot\bm{I}_r)).
\end{equation}
We approximate $\ell(\bm{Z},\bm{D},\bm{X})$ with its second order Taylor expansion around $\bm{D}_0$, i.e.,
\begin{equation}
\begin{aligned}
\ell(\bm{Z},\bm{D},\bm{X})=&\ell(\bm{Z},\bm{D}_0,\bm{X})+\langle\bm{g}_{\scalebox{.6}{$\bm{D}$}}, \bm{D}-\bm{D}_0\rangle\\
&+\tfrac{1}{2}\mbox{vec}(\bm{D}-\bm{D}_0)^T\bm{H}_{\scalebox{.6}{$\bm{D}$}}\mbox{vec}(\bm{D}-\bm{D}_0)+R_0,
\end{aligned}
\end{equation}
where $R_0=O(\tfrac{\Vert \ell^{(3)}\Vert}{6})$ denotes the residual of the approximation and $\bm{H}\in\mathbb{R}^{r^2\times r^2}$ denotes the Hessian matrix. We have
\begin{equation}
\bm{H}=\left[ 
\begin{matrix}
\bm{H}_{\scalebox{.6}{$\bm{D}$}} & \bm{0} & \cdots & \bm{0}\\
\bm{0} & \bm{H}_{\scalebox{.6}{$\bm{D}$}} & \cdots & \bm{0}\\
\vdots & \cdots &\ddots & \vdots\\
\bm{0} & \bm{0} &\cdots & \bm{H}_{\scalebox{.6}{$\bm{D}$}}
\end{matrix}
\right],
\end{equation}
where
\begin{equation}
\bm{H}_{\scalebox{.6}{$\bm{D}$}}:= \bm{Z}\bm{Z}^T\odot\bm{W}_2+\alpha\bm{W}_2\odot\bm{I}_r.
\end{equation}
One has $\mbox{vec}(\bm{D}-\bm{D}_0)^T\bm{H}\mbox{vec}(\bm{D}-\bm{D}_0)=\text{Tr}((\bm{D}-\bm{{D}_0})\bm{H}(\bm{D}-\bm{{D}_0})^T)$. Denote
\begin{equation}
\begin{aligned}
\ell'(\bm{Z},\bm{D},\bm{X})=&\ell(\bm{Z},\bm{D}_0,\bm{X})+\langle\bm{g}_{\scalebox{.6}{$\bm{D}$}}, \bm{D}-\bm{D}_0\rangle\\
+&\tfrac{\tau}{2}\text{Tr}((\bm{D}-\bm{{D}_0})\bm{H}_{\scalebox{.6}{$\bm{D}$}}(\bm{D}-\bm{{D}_0})^T),
\end{aligned}
\end{equation}
where $\tau>1$. Since $\bm{H}_{\scalebox{.6}{$\bm{D}$}}$ is positive definite, we have
\begin{equation}\label{Eq.supp_D01}
\ell(\bm{Z},\bm{D},\bm{X})\leq\ell'(\bm{Z},\bm{D},\bm{X}),
\end{equation}
provided that $\tau$ is large enough. We then minimize $\ell'$ by letting the derivative be zero and get
\begin{equation}\label{Eq.supp_D02}
\bm{D}=\bm{D}_0-\bm{\Delta}_{\scalebox{.6}{$\bm{D}$}}.
\end{equation}
where $\bm{\Delta}_{\scalebox{.6}{$\bm{D}$}}=\tfrac{1}{\tau}\bm{g}_{\scalebox{.6}{$\bm{D}$}}\bm{H}_{\scalebox{.6}{$\bm{D}$}}^{-1}$. Invoking (\ref{Eq.supp_D02}) into (\ref{Eq.supp_D01}), we have
\begin{equation}
\ell(\bm{Z},\bm{D}_0-\bm{\Delta}_{\scalebox{.6}{$\bm{D}$}},\bm{X})\leq \ell(\bm{Z},\bm{D}_0,\bm{X})-\tfrac{1}{2\tau}\text{Tr}(\bm{g}_{\scalebox{.6}{$\bm{D}$}}\bm{H}_{\scalebox{.6}{$\bm{D}$}}^{-1}\bm{g}_{\scalebox{.6}{$\bm{D}$}}^T).
\end{equation}
\end{proof}

\subsection{Proof for Lemma 4}
\noindent \textbf{Lemma 4.}
\textit{$\Vert\bm{X}(\bm{Z}^T\odot\bm{K}_{\scalebox{.6}{$\bm{X}\bm{D}_1$}})-\bm{X}(\bm{Z}^T\odot\bm{K}_{\scalebox{.6}{$\bm{X}\bm{D}_2$}})\Vert_F\leq\tfrac{c}{\sigma\sqrt{n}}\Vert \bm{X}\Vert_2\Vert\bm{D}_1-\bm{D}_2\Vert_F$, where $c$ is a small constant.}

\begin{proof}
Since $\bm{Z}=\min\limits_{\bm{Z}}\tfrac{1}{2}\Vert \phi(\bm{X})-\phi(\bm{D})\bm{Z}\Vert_F^2+\tfrac{\beta}{2}\Vert \bm{Z}\Vert_F^2$, we have
\begin{equation}
\bm{Z}=(\phi(\bm{D})^T\phi(\bm{D})+\beta\bm{I}_r)^{-1}\phi(\bm{D})^T\phi(\bm{X}).
\end{equation}
Denote $\phi(\bm{D})=\bm{U}\bm{S}\bm{V}^T=\bm{U}\text{diag}(\lambda_1,\cdots,\lambda_r)\bm{V}^T$ (the singular value decomposition), we have $\phi(\bm{X})=\bm{U}\hat{\bm{S}}\hat{\bm{V}}^T=\bm{U}\text{diag}(\hat{\lambda}_1,\cdots,\hat{\lambda}_r)\hat{\bm{V}}^T$ because $\phi(\bm{X})$ and $\phi(\bm{D})$ have the same column basis. Then
\begin{equation}
\begin{aligned}
\bm{Z}=&\bm{V}(\bm{S}^2+\beta\bm{I})^{-1}\bm{S}\hat{\bm{S}}\hat{\bm{V}}^T\\
&=\bm{V}\text{diag}(\tfrac{\lambda_1\hat{\lambda}_1}{\lambda_1^2+\beta},\cdots,\tfrac{\lambda_r\hat{\lambda}_r}{\lambda_r^2+\beta})\hat{\bm{V}}^T.
\end{aligned}
\end{equation}
Suppose $\beta$ is large enough, we have $\tfrac{\lambda_i\hat{\lambda}_i}{\lambda_i^2+\beta}<1$ for $i=1,\cdots,r$. It follows that $\Vert \bm{Z}\Vert_F^2<r$ and $E[z_{ij}^2]<\tfrac{1}{n}$, which indicates that
\begin{equation}
\left\{
\begin{matrix}
\sigma_z=E[\vert z_{ij}-\mu_z\vert]<\tfrac{1}{\sqrt{n}},\\
-\tfrac{1}{\sqrt{n}}<\mu_z=E[z_{ij}]<\tfrac{1}{\sqrt{n}}.
\end{matrix}
\right.
\end{equation}
According to Chebyshev's inequality, we have
\begin{equation}
\text{Pr}(\vert z_{ij}\vert>\tfrac{c_0}{\sqrt{n}}+\tfrac{1}{\sqrt{n}})<\tfrac{1}{c_0^2}.
\end{equation}
Therefore, $\vert z_{ij}\vert<\tfrac{c_0}{\sqrt{n}}$ holds with high probability provided that $c_0$ is large enough. Suppose $z_{ij}\sim \mathcal{N}(\mu_z,\sigma_z^2)$, we have 
\begin{equation}
\text{Pr}(\vert z_{ij}\vert>\tfrac{c_0}{\sqrt{n}})<e^{-0.5c_0^2}
\end{equation}
according to the upper bound of Q-function of normal distribution. Then using union bound, we obtain
\begin{equation}
\text{Pr}\big(\vert z_{ij}\vert<\tfrac{c_0}{\sqrt{n}},\forall (i,j)\big)<1-nre^{-0.5c_0^2},
\end{equation}
It is equivalent to
\begin{equation}
\text{Pr}\big(\vert z_{ij}\vert<\tfrac{c_1}{\sqrt{n}},\forall (i,j)\big)<1-\tfrac{1}{(nr)^{c_0-1}},
\end{equation}
where $c_1=\sqrt{2c_0\log(nr)}$.

On the other hand, the partial gradient of entry $(i,j)$ of $\bm{K}_{\scalebox{.6}{$\bm{X}\bm{D}$}}$ in terms of $\bm{D}_{:j}$ (the $j$-th column of $\bm{D}$) can be given by
\begin{equation}
\tfrac{\partial \bm{K}_{\scalebox{.6}{$\bm{X}\bm{D}$}}(i,j)}{\partial \bm{D}_{:j}}=-\tfrac{1}{\sigma^2}(\bm{X}_{:i}-\bm{D}_{:j})\exp(-\tfrac{\Vert \bm{X}_{:i}-\bm{D}_{:j}\Vert^2}{2\sigma^2}).
\end{equation}
Because $\vert x\exp(\tfrac{-x^2}{2\sigma^2})\vert\leq\sigma \exp(-0.5)<0.61\sigma$,  we have $\vert \tfrac{\partial \bm{K}_{\scalebox{.6}{$\bm{X}\bm{D}$}}(i,j)}{\partial \bm{D}_{kj}}\vert<\tfrac{c_2}{\sigma}$ for some constant $c_2$. Then
\begin{equation}
\Vert\bm{K}_{\scalebox{.6}{$\bm{X}\bm{D}_1$}}-\bm{K}_{\scalebox{.6}{$\bm{X}\bm{D}_2$}}\Vert_F\leq \tfrac{c_3}{\sigma}\Vert\bm{D}_1-\bm{D}_2\Vert_F
\end{equation}
some small constant $c_3$.

According to the above analysis, we get
\begin{equation}
\begin{aligned}
&\Vert\bm{X}(\bm{Z}^T\odot\bm{K}_{\scalebox{.6}{$\bm{X}\bm{D}_1$}})-\bm{X}(\bm{Z}^T\odot\bm{K}_{\scalebox{.6}{$\bm{X}\bm{D}_2$}})\Vert_F\\
\leq &\Vert \bm{X}\Vert_2\Vert\bm{Z}^T\odot(\bm{K}_{\scalebox{.6}{$\bm{X}\bm{D}_1$}}-\bm{K}_{\scalebox{.6}{$\bm{X}\bm{D}_2$}})\Vert_F\\
\leq &\tfrac{c_1}{\sqrt{n}}\Vert \bm{X}\Vert_2\Vert\bm{K}_{\scalebox{.6}{$\bm{X}\bm{D}_1$}}-\bm{K}_{\scalebox{.6}{$\bm{X}\bm{D}_2$}}\Vert_F\\
\leq &\tfrac{c_1}{\sqrt{n}}\tfrac{c_3}{\sigma}\Vert \bm{X}\Vert_2\Vert\bm{D}_1-\bm{D}_2\Vert_F\\
=&\tfrac{c}{\sigma\sqrt{n}}\Vert \bm{X}\Vert_2\Vert\bm{D}_1-\bm{D}_2\Vert_F.
\end{aligned}
\end{equation}
\end{proof}

\subsection{Proof for Lemma 5}
\noindent \textbf{Lemma 5.}
\textit{For sufficiently small $\eta$, Algorithm 1 converges to a stationary point.}

\begin{proof}
Denote the objective function of (7) by $\ell(\bm{Z},\bm{D},\bm{X})$, which is lower-bounded by at least 0. When $\eta=0$, as the three subproblems are well addressed and do not diverge, we have
$\ell(\bm{Z}_{t+1},\bm{D}_{t+1},\bm{X}_{t+1})<\ell(\bm{Z}_{t+1},\bm{D}_{t+1},\bm{X}_{t})<\ell(\bm{Z}_{t+1},\bm{D}_{t},\bm{X}_{t})<\ell(\bm{Z}_{t},\bm{D}_{t},\bm{X}_{t})$. It indicates that $\Delta_t=\ell(\bm{Z}_{t},\bm{D}_{t},\bm{X}_{t})-\ell(\bm{Z}_{t+1},\bm{D}_{t+1},\bm{X}_{t+1})\rightarrow 0$ when $t\rightarrow \infty$. When $\Delta_t=0$, the gradient of $\ell(\bm{Z}_{t},\bm{D}_{t},\bm{X}_{t})$ is 0. 
Then Algorithm 1 converges to a stationary point. 

When $\eta>0$ and take $\bm{D}$ as an example, because $\bm{\Delta}_{D,t}$ is not exact enough, we decompose $\bm{\Delta}_{D,t}$ as $\bm{\Delta}_{D,t}=c_t\bm{\Delta}_{D,t}^*+\bm{\Delta}_{D,t}'$, where $0<c_t<1$ and $\bm{\Delta}_{D,t}^*$ is nearly optimal at iteration $t$. Similarly, we have $\bm{\Delta}_{D,t-1}=c_{t-1}\bm{\Delta}_{D,t}^*+\bm{\Delta}_{D,t-1}'$. Then $\widehat{\bm{\Delta}}_t=(c_t+c_{t-1}\eta+\cdots+c_{0}\eta^{t})\bm{\Delta}_{D,t}^*+\bm{\epsilon}'$, where $\bm{\epsilon}'=\sum_{i=0}^t\eta^i\bm{\Delta}_{D,i}'$. $\bm{\epsilon}'$ could be small compared to $\bm{\Delta}_{D,t}'$ because the signs of elements of $\bm{\Delta}_{D,0}',\cdots,\bm{\Delta}_{D,t}'$ may change. Suppose $c_t$ and $\eta$ are small enough such that $c_t<c_t+c_{t-1}\eta+\cdots+c_{0}\eta^{t}<1$, then $\widehat{\bm{\Delta}}_t$ is closer than $\bm{\Delta}_t$ to $\bm{\Delta}_t^*$.  It indicates $\ell(\bm{Z}_{t+1},\bm{D}_{t}-\widehat{\bm{\Delta}}_t,\bm{X}_{t})<\ell(\bm{Z}_{t+1},\bm{D}_{t}-{\bm{\Delta}}_t,\bm{X}_{t})<\ell(\bm{Z}_{t+1},\bm{D}_{t},\bm{X}_{t})$. That is why the momentum can accelerate the convergence. 

More formally, take $\bm{D}$ with polynomial kernel as an example, in Lemma 3, we have proved 
$\ell(\bm{Z},\bm{D}-\bm{\Delta}_{\scalebox{.6}{$\bm{D}$}},\bm{X})-\ell(\bm{Z},\bm{D},\bm{X})\leq-\tfrac{1}{2\tau}\textup{Tr}(\bm{g}_{\scalebox{.6}{$\bm{D}$}}\bm{H}_{\scalebox{.6}{$\bm{D}$}}^{-1}\bm{g}_{\scalebox{.6}{$\bm{D}$}}^T)$. As $\bm{\Delta}_{\scalebox{.6}{$\bm{D}$}}=\tfrac{1}{\tau}\bm{g}_{\scalebox{.6}{$\bm{D}$}}\bm{H}_{\scalebox{.6}{$\bm{D}$}}^{-1}$, we have
\begin{equation*}
\ell(\bm{Z},\bm{D}-\bm{\Delta}_{\scalebox{.6}{$\bm{D}$}},\bm{X})-\ell(\bm{Z},\bm{D},\bm{X})\leq-\tfrac{\tau}{2}\textup{Tr}(\bm{\Delta}_{\scalebox{.6}{$\bm{D}$}}\bm{H}\bm{\Delta}_{\scalebox{.6}{$\bm{D}$}}^T).
\end{equation*}
When momentum is used, $\bm{\Delta}_{\scalebox{.6}{$\bm{D}$}}$ is replaced by $\bm{\Delta}_{\scalebox{.6}{$\bm{D}$}}+\eta\widehat{\bm{\Delta}}_{\scalebox{.6}{$\bm{D}$}}$. Using the Taylor approximation similar to Lemma 3, we have
\begin{equation}
\begin{aligned}
&\ell(\bm{Z},\bm{D}-\bm{\Delta}_{\scalebox{.6}{$\bm{D}$}}-\eta\widehat{\bm{\Delta}}_{\scalebox{.6}{$\bm{D}$}},\bm{X})\\
\leq&\ell(\bm{Z},\bm{D}-\bm{\Delta}_{\scalebox{.6}{$\bm{D}$}},\bm{X})+\langle\bm{G}_{\eta}, \eta\widehat{\bm{\Delta}}_{\scalebox{.6}{$\bm{D}$}}\rangle+\tfrac{\eta^2\tau}{2}\textup{Tr}(\widehat{\bm{\Delta}}_{\scalebox{.6}{$\bm{D}$}}\bm{H}\widehat{\bm{\Delta}}_{\scalebox{.6}{$\bm{D}$}}^T),
\end{aligned}
\end{equation}
where $\bm{G}_{\eta}$ denotes the partial derivative of $\ell$ at $\bm{D}-\bm{\Delta}_{\scalebox{.6}{$\bm{D}$}}$. It follows that
\begin{equation}
\begin{aligned}
&\ell(\bm{Z},\bm{D}-\bm{\Delta}_{\scalebox{.6}{$\bm{D}$}}-\eta\widehat{\bm{\Delta}}_{\scalebox{.6}{$\bm{D}$}},\bm{X})\\
\leq&\ell(\bm{Z},\bm{D}-\bm{\Delta}_{\scalebox{.6}{$\bm{D}$}},\bm{X})+\eta^2\tau\textup{Tr}(\widehat{\bm{\Delta}}_{\scalebox{.6}{$\bm{D}$}}\bm{H}\widehat{\bm{\Delta}}_{\scalebox{.6}{$\bm{D}$}}^T).
\end{aligned}
\end{equation}
If $\eta\widehat{\bm{\Delta}}_{\scalebox{.6}{$\bm{D}$}}$ is a descent value, we have 
\begin{equation}
\begin{aligned}
&\ell(\bm{Z},\bm{D}-\bm{\Delta}_{\scalebox{.6}{$\bm{D}$}}-\eta\widehat{\bm{\Delta}}_{\scalebox{.6}{$\bm{D}$}},\bm{X})\\
<&\ell(\bm{Z},\bm{D}-\bm{\Delta}_{\scalebox{.6}{$\bm{D}$}},\bm{X})\\
\leq&\ell(\bm{Z},\bm{D},\bm{X})-\tfrac{\tau}{2}\textup{Tr}(\bm{\Delta}_{\scalebox{.6}{$\bm{D}$}}\bm{H}\bm{\Delta}_{\scalebox{.6}{$\bm{D}$}}^T).
\end{aligned}
\end{equation}
Otherwise, we have 
\begin{equation}
\begin{aligned}
&\ell(\bm{Z},\bm{D}-\bm{\Delta}_{\scalebox{.6}{$\bm{D}$}}-\eta\widehat{\bm{\Delta}}_{\scalebox{.6}{$\bm{D}$}},\bm{X})\\
\leq&\ell(\bm{Z},\bm{D},\bm{X})-\tfrac{\tau}{2}\textup{Tr}(\bm{\Delta}_{\scalebox{.6}{$\bm{D}$}}\bm{H}\bm{\Delta}_{\scalebox{.6}{$\bm{D}$}}^T)+\eta^2\tau\textup{Tr}(\widehat{\bm{\Delta}}_{\scalebox{.6}{$\bm{D}$}}\bm{H}\widehat{\bm{\Delta}}_{\scalebox{.6}{$\bm{D}$}}^T).
\end{aligned}
\end{equation}
Since $\bm{H}$ is positive definite, we have
\begin{equation}
\ell(\bm{Z},\bm{D}-\bm{\Delta}_{\scalebox{.6}{$\bm{D}$}}-\eta\widehat{\bm{\Delta}}_{\scalebox{.6}{$\bm{D}$}},\bm{X})\leq\ell(\bm{Z},\bm{D},\bm{X})
\end{equation}
if $\eta$ is small enough. Then similar to the case of $\eta=0$, the convergence can be proved.
\end{proof}

\subsection{Proof for Lemma 6}
\noindent \textbf{Lemma 6.}
\textit{Updating $\bm{D}$ as $\bm{D}-\bm{\Delta}_{\scalebox{.6}{$\bm{D}$}}$ does not diverge and $\hat{\ell}(\bm{z},[\bm{x}]_{\bar{\omega}},\bm{D}-\bm{\Delta}_{\scalebox{.6}{$\bm{D}$}})-\hat{\ell}(\bm{z},[\bm{x}]_{\bar{\omega}},\bm{D})\leq -\tfrac{1}{2\tau\tau_0}\Vert\nabla_{\scalebox{.6}{$\bm{D}$}}\hat{\ell}\Vert_F^2$ provided that $\tau>1$, where $\tau_0=\Vert \bm{z}\bm{z}^T\odot\bm{W}_2+\alpha\bm{W}_2\odot\bm{I}_r\Vert_2$.}

\begin{proof}
Fixing $\bm{W}_1$ and $\bm{W}_2$, we have $\Vert\nabla_{\scalebox{.6}{$\bm{D}$}_1}\hat{\ell}-\nabla_{\scalebox{.6}{$\bm{D}$}_2}\hat{\ell}\Vert_F\leq \Vert q\bm{z}\bm{z}^T\odot\bm{W}_2+q\alpha\bm{W}_2\odot\bm{I}_r\Vert_2\Vert \bm{D}_1-\bm{D}_2\Vert_F$, which means the Lipschitz constant of $\hat{\ell}$'s gradient can be estimated as $\tau_0=\Vert q\bm{z}\bm{z}^T\odot\bm{W}_2+q\alpha\bm{W}_2\odot\bm{I}_r\Vert_2$. It follows that
\begin{equation}\label{Eq_supp_online_D1}
\begin{aligned}
\hat{\ell}(\bm{z},[\bm{x}]_{\bar{\omega}},\bm{D})\leq&\hat{\ell}(\bm{z},[\bm{x}]_{\bar{\omega}},\bm{D}_0)+\langle \nabla_{\scalebox{.6}{$\bm{D}$}}\hat{\ell}, \bm{D}-\bm{D}_0 \rangle\\
&+\tfrac{\tau\tau_0}{2}\Vert \bm{D}-\bm{D}_0\Vert_F^2,
\end{aligned}
\end{equation}
where $\tau>1$. We minimize the right part of (\ref{Eq_supp_online_D1}) and get 
\begin{equation}\label{Eq_supp_online_D2}
\bm{D}=\bm{D}_0-\tfrac{1}{\tau\tau_0}\nabla_{\scalebox{.6}{$\bm{x}$}}\hat{\ell}:= \bm{D}_0-\bm{\Delta}_{\scalebox{.6}{$\bm{D}$}}.
\end{equation}
Substituting (\ref{Eq_supp_online_D2}) into (\ref{Eq_supp_online_D1}), we have
\begin{equation}
\hat{\ell}(\bm{z},[\bm{x}]_{\bar{\omega}},\bm{D}_0-\bm{\Delta}_{\scalebox{.6}{$\bm{D}$}})-\hat{\ell}(\bm{z},[\bm{x}]_{\bar{\omega}},\bm{D}_0)\leq -\tfrac{1}{2\tau\tau_0}\Vert \nabla_{\scalebox{.6}{$\bm{D}$}}\hat{\ell}\Vert_F^2.
\end{equation}
\end{proof}

\subsection{Derivation for (36)}
As the number of observed entries in each column of $\bm{X}$ is $o_{\scalebox{.6}{$\bm{X}$}}=\rho m$, the number of observed entries in each column of $\phi(\bm{X})\in\mathbb{R}^{\bar{m}\times n}$ is 
\begin{equation}\label{Eq.sr_0}
o_{\phi(\bm{x})}=\tbinom{\rho m+q}{q},
\end{equation}
where $\phi$ is a $q$-order polynomial map. It is known that the number of observed entries in $\phi(\bm{X})$ should be larger than the number of degrees of freedom of $\phi(\bm{X})$, otherwise it is impossible to determine $\phi(\bm{X})$ uniquely among all rank-$r$ matrices of size $\bar{m}\times n$ \cite{pmlr-v70-ongie17a}. Then we require
\begin{equation}\label{Eq.sr_1}
no_{\phi(\bm{x})}>nr+(\bar{m}-r)r,
\end{equation}
where $\bar{m}=\binom{m+q}{q}$ and $r=\binom{d+pq}{pq}$. Substituting (\ref{Eq.sr_0}) into (\ref{Eq.sr_1}) and dividing both sides with $\bar{m}$, we get
\begin{equation}
\tfrac{\binom{\rho m+q}{q}}{\binom{m+q}{q}}>\dfrac{nr+(\bar{m}-r)r}{n\bar{m}}.
\end{equation}
Since 
\begin{equation}
\tfrac{\binom{\rho m+q}{q}}{\binom{m+q}{q}}=\tfrac{(\rho m+q)(\rho m+q-1)\cdots(\rho m+1)}{(m+q)(m+q-1)\cdots(m+1)},
\end{equation}
we have $\binom{\rho m+q}{q}/\binom{m+q}{q}\approx \rho^q$ for small $q$. It follows that
\begin{equation}\label{Eq.rho}
\begin{aligned}
\rho>&\bigl(\tfrac{nr+(\bar{m}-r)r}{n\bar{m}}\bigr)^{\tfrac{1}{q}}=\bigl(\tfrac{r}{n}+\tfrac{r}{\bar{m}}-\tfrac{r^2}{n\bar{m}}\bigr)^{\tfrac{1}{q}}\\
=&\bigl(\tfrac{u\binom{d+pq}{pq}}{n}+\tfrac{u\binom{d+pq}{pq}}{\binom{m+q}{q}}-\tfrac{u^2\binom{d+pq}{pq}^2}{n\binom{m+q}{q}}\bigr)^{\tfrac{1}{q}}\\
:=&\kappa(m,n,d,p,q,u).
\end{aligned}
\end{equation}

\subsection{Derivation for (37)}
We reformulate RBF kernel as
\begin{equation}\label{Eq.RBF2poly}
\begin{aligned}
k(\bm{x},\bm{y})&=\exp\bigl(-\tfrac{1}{2\sigma^2}(\Vert \bm{x}\Vert^2+\Vert \bm{y}\Vert^2)\bigr)\exp\left(\tfrac{1}{\sigma^2}\langle \bm{x},\bm{y}\rangle\right)\\
&:= C\sum_{k=0}^\infty\dfrac{\langle \bm{x},\bm{y}\rangle^k}{\sigma^{2k}k!}\\
&=C\sum_{k=0}^q\dfrac{\langle \bm{x},\bm{y}\rangle^k}{\sigma^{2k}k!}+O(\dfrac{c^{q+1}}{(q+1)!}),
\end{aligned}
\end{equation}
where $0<c<1$ provided that $\sigma^2>\vert \bm{x}^T\bm{y}\vert$ and $C=\exp\bigl(-\tfrac{1}{2\sigma^2}(\Vert \bm{x}\Vert^2+\Vert \bm{y}\Vert^2)\bigr)$. We see that RBF kernel can be approximated by a weighted sum of polynomial kernels with orders $0,1,\cdots,q$, where the error is $O(\tfrac{c^{q+1}}{(q+1)!})$. The feature map of the weighted sum is a $q$-order polynomial map, denoted by $\hat{\phi}$. Then it follows from (\ref{Eq.RBF2poly}) that
\begin{equation}
\phi(\bm{x})^T\phi(\bm{x})=\hat{\phi}(\bm{x})^T\hat{\phi}(\bm{x})+O(\tfrac{c^{q+1}}{(q+1)!}),
\end{equation}
and further
\begin{equation}
\phi_i(\bm{x})=\hat{\phi}_i(\bm{x})+O(\sqrt{\tfrac{c^{q+1}}{(q+1)!}}),
\end{equation}
in which we have assumed that the signs of $\phi_i(\bm{x})$ and $\hat{\phi}_i(\bm{x})$ are the same because it has no influence on the feature map.
It means the feature map $\phi$ of RBF kernel can be well approximated by a $q$-order polynomial map, where the approximation error is $O(\sqrt{\tfrac{c^{q+1}}{(q+1)!}})$ and could be nearly zero. Therefore, $\rho>\kappa(m,n,d,p,q,u)$ in (\ref{Eq.rho}) holds for RBF kernel with error $O(\sqrt{\tfrac{c^{q+1}}{(q+1)!}})$ in recovering $\phi(\bm{X})$. When $\phi(\bm{X})$ is recovered, $\bm{X}$ is naturally recovered because $\bm{X}$ itself is the first-order feature in $\phi(\bm{X})$.

\section{More about the experiments}
\subsection{An intuitive example}
We use a simple example of nonlinear data to intuitively show the performance of our high-rank matrix completion method KFMC. Specifically, we sample 100 data points from the following twisted cubic function
\begin{equation}\label{Eq_twistedF}
x_1=s,\ x_2=s^2,\ x_3=s^3,
\end{equation}
where $s\sim\mathcal{U}(-1,1)$. Then we obtain a $3\times 100$ matrix, which is of full-rank. For each data point (column), we randomly remove one entry. The recovery results of low-rank matrix completion and our KFMC are shown in Figure \ref{Fig_syn_twisted_1}. We see that LRMC absolutely failed because it cannot handle full-rank matrix. On the contrary, our KFMC recovered the missing entries successfully. The performance of KFMC at different iteration is shown in Figure \ref{Fig_syn_twisted_2}, which demonstrated that KFMC shaped the data into the curve gradually. It is worth mentioning that when we remove two entries of each column of the matrix,  KFMC cannot recovery the missing entries because the number of observed entries is smaller than the latent dimension of the data.

\begin{figure}[h!]
\centering
\includegraphics[width=9cm,trim={40 0 0 0},clip]{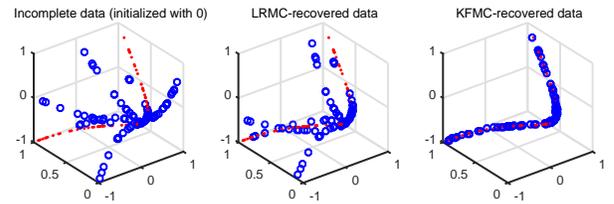}
\vspace{-15pt}
\caption{Recovery result on data drawn from (\ref{Eq_twistedF}) (the red points are the complete data)}
\label{Fig_syn_twisted_1}
\end{figure}
\vspace{-0pt}

\begin{figure}[h!]
\centering
\includegraphics[width=9cm,trim={40 0 0 0},clip]{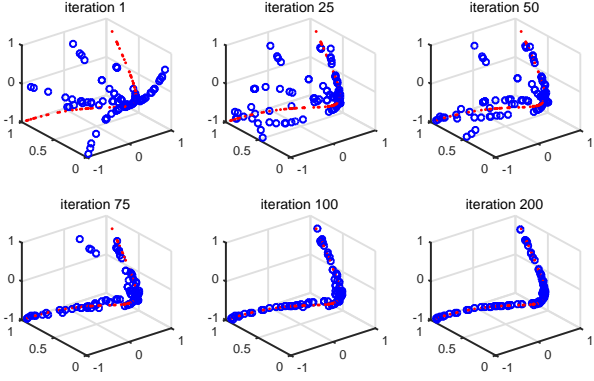}
\vspace{-25pt}
\caption{KFMC recovery performance at different iteration on data drawn from (\ref{Eq_twistedF})}
\label{Fig_syn_twisted_2}
\end{figure}

\subsection{Compared methods and parameter settings}
For offline matrix completion, our KFMC with polynomial kernel and KFMC with RBF kernel are compared with the following methods.
\begin{itemize}
\item[]\textbf{LRF} (low-rank factorization based matrix completion \cite{sun2016guaranteed}). LRF is solved by alternating minimization. The matrix rank and regularization parameter are searched within a broad range such that the best results are reported in our experiments. 
\item[]\textbf{NNM} (nuclear norm minimization \cite{CandesRecht2009}). NNM is solved by inexact augmented lagrange multiplier \cite{LinChenMa2013} and has no parameter to determine beforehand. Therefore it is good baseline for our evaluation.
\item[]\textbf{VMC} (algebraic variety model for matrix completion \cite{pmlr-v70-ongie17a}). In VMC, second-order polynomial kernel is used, where the hyper-parameter is chosen from $\lbrace 1,10,100\rbrace$. The parameter of Schatten-$p$ norm is set as 0.5, which often performs the best. To reduce its computational cost, randomized SVD  \cite{randomsvd}, instead of full SVD, is performed.
\item[]\textbf{NLMC} (nonlinear matrix completion \cite{FANNLMC}). In NLMC, RBF kernel is used. The parameter $\sigma$ of the kernel is chosen from $\lbrace 0.5\bar{d},1\bar{d},3\bar{d}\rbrace$, where $\bar{d}$ is the average distance of all pair-wise data points. The parameter of Schatten-$p$ norm is set as 0.5 and randomized SVD is also performed.
\end{itemize}
In our KFMC(Poly) method, second order polynomial kernel is used, where the hyper-parameter is set as 1. The regularization parameters $\alpha$ and $\beta$ are chosen from $\lbrace 0.01,0.1\rbrace$. In our KFMC(RBF), the setting of parameter $\sigma$ is the same as that of NLMC. The regularization parameter $\beta$ is chosen from $\lbrace 0.001,0.0001\rbrace$ while $\alpha$ does not matter. The parameter $r$ of KFMC(Poly) and KFMC(RBF) are chosen from $\lbrace 0.5m,1m,2m\rbrace$, where $m$ is the row dimension of the matrix. 

For online matrix completion, the parameter setting of OL-KFMC is similar to that of KFMC. Our OL-KFMC(Poly) and OL-KFMC(RBF) are compared with the following methods.
\begin{itemize}
\item[]\textbf{GROUSE} \cite{balzano2010online}\footnote{http://web.eecs.umich.edu/$\sim$girasole/?p=110}. The learning rate and matrix rank are searched within large ranges to provide the best performances.
\item[]\textbf{OL-DLSR} (online dictionary learning and sparse representation based matrix completion). OL-DLSR is achieved by integrating \cite{mairal2009online} with \cite{FAN2018SFMC}. It solves the following problem
\begin{equation}
\mathop{\text{minimize}}_{\bm{D}\in\mathcal{C},\bm{z}}\tfrac{1}{2}\Vert \bm{\omega}\odot(\bm{x}-\bm{D}\bm{z})\Vert^2+\lambda\Vert \bm{z}\Vert_1
\end{equation}
for a set of incomplete data columns $\lbrace\bm{x}\rbrace$. $\bm{\omega}$ is a binary vector with $\omega_i=0$ if entry $i$ of $\bm{x}$ is missing and $\omega_i=1$ otherwise.  According to \cite{FAN2018SFMC}, the method can recover high-rank matrices online when the data are drawn from a union of subspaces. We determine $\lambda$ and the number of columns of $\bm{D}$ carefully to give the best performances of OL-DLSR in our experiments.
\item[]\textbf{OL-LRF} (online LRF \cite{sun2016guaranteed,jin2016provable}). OL-LRF is similar to OL-DLSR. The only difference is that $\Vert \bm{z}\Vert_1$ is replaced by $\dfrac{1}{2}\Vert \bm{z}\Vert_F^2$. In OL-LRF, the matrix rank is carefully determined to give the best performances in our experiments.
\end{itemize}

For out-of-sample extension of matrix completion, our OSE-KFMC is compared with the following two methods.
\begin{itemize}
\item[]\textbf{OSE-LRF} First, perform SVD on a complete training data matrix, i.e., $\bm{X}=\bm{U}\bm{S}\bm{V}^T$, where $\bm{U}\in\mathbb{R}^{m\times r}$, $\bm{S}\in\mathbb{R}^{r\times r}$, $\bm{V}\in\mathbb{R}^{n\times r}$, and $r=\text{rank}(\bm{X})$. For a new incomplete data column $\bm{x}$, the missing entries are recovered as
\begin{equation}
\bm{x}_{\bar{\omega}}=\bm{U}_{\bar{\omega}}(\bm{U}_{\omega}^T\bm{U}_{\omega}+\lambda\bm{I}_{\vert\omega\vert})^{-1}\bm{U}_{\omega}^T\bm{x}_{\omega},
\end{equation}
where $\omega$ denotes the locations of observed entries, $\bar{\omega}$ denotes the locations of missing entries, $\lambda$ is a small constant, and $\bm{U}_{\bar{\omega}}$ consists of the rows of $\bm{U}$ corresponding to $\bar{\omega}$.
\item[]\textbf{OSE-DLSR} First, a dictionary $\bm{D}$ is learned by the method of \cite{mairal2009online} from the training data. Given a new incomplete data $\bm{x}$, we can obtain the sparse coefficient as 
\begin{equation}
\bm{z}=\min\limits_{\bm{z}}\tfrac{1}{2}\Vert\bm{\omega}\odot(\bm{x}-\bm{D}\bm{z})\Vert^2+\lambda\Vert\bm{z}\Vert_1.
\end{equation}
Finally, the missing entries of $\bm{x}$ can be recovered as $\bm{x}_{\bar{\omega}}=\bm{D}_{\bar{\omega}}\bm{z}$.
\end{itemize}

The experiments are conducted with MATLAB on a computer with Inter-i7-3.4GHz Core and 16 GB RAM. The maximum iteration of each offline matrix completion method is 500, which is often enough to converge or give a high recovery accuracy. It also provides a baseline to compare the computational costs of VMC, NLMC, and KFMC.

\subsection{Synthetic data}
Take the case of three nonlinear subspaces as an example, the optimization curves of our KFMC with different momentum parameter $\eta$ are shown in Figure \ref{Fig_syn_off_iter}. We see that a larger $\eta$ can lead to a faster convergence. Particularly, compared with KFMC(Poly), KFMC(RBF) requires fewer iterations to converge, while in each iteration the computational cost of the former is a little bit higher than that of the latter. In this paper, we set $\eta=0.5$ for all experiments.

\begin{figure}[h!]
\centering
\includegraphics[width=8cm]{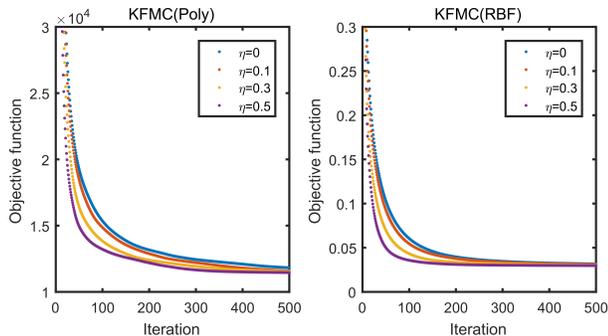}
\caption{Optimization of KFMC with different momentum parameter $\eta$}
\label{Fig_syn_off_iter}
\end{figure}

Figure \ref{Fig_syn_on_cost} shows the online KFMC's iterative changes of empirical cost function
\begin{equation}\label{Eq.mc_online_emperical_2}
g_t(\bm{D}):=\dfrac{1}{t}\sum_{j=1}^t\ell([\bm{x}_j]_{\omega_j},\bm{D})
\end{equation}
and empirical recovery error
\begin{equation}\label{Eq.mc_online_emperical_error}
e_t(\bm{x}):=\dfrac{1}{t}\sum_{j=1}^t\dfrac{\Vert\bm{x}_j-\hat{\bm{x}}_j\Vert}{\Vert\bm{x}_j\Vert},
\end{equation}
where $\hat{\bm{x}}_j$ denotes the recovered column and $t$ is the number of online samples. At the beginning of the online learning ($t$ is small), the recover errors and the values of cost function are high. With the increasing of $t$, the recover errors and the values of cost function decreased. In practice, we can re-pass the data to reduce the recovery errors. In addition, when $t$ is large enough and the structure of the data is assumed to be fixed, we do not need to update $\bm{D}$. If the data structure changes according to time, we can just update $\bm{D}$ all the time in order to adapt to the changes.

\begin{figure}[h!]
\centering
\includegraphics[width=8cm]{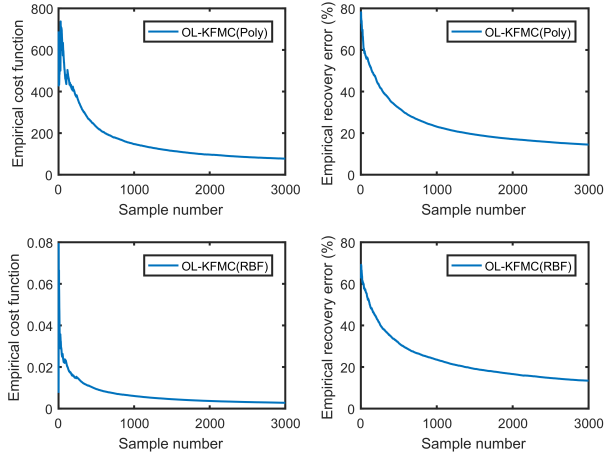}
\caption{Empirical cost function and recovery error of online KFMC}
\label{Fig_syn_on_cost}
\end{figure}

In our experiments of online matrix completion, the reported recovery errors are the results after the data matrix was passed for 10 times. Figure \ref{Fig_syn_on_pass} shows the matrix completion errors of different number of passes. Our OL-KFMC(Poly) and OL-KFMC(RBF) have the lowest recovery errors. The recovery errors of OL-LRF and GROUSE are considerably high because they are low-rank methods but the matrix in the experiment is of high-rank.

\begin{figure}[h!]
\centering
\includegraphics[width=7cm]{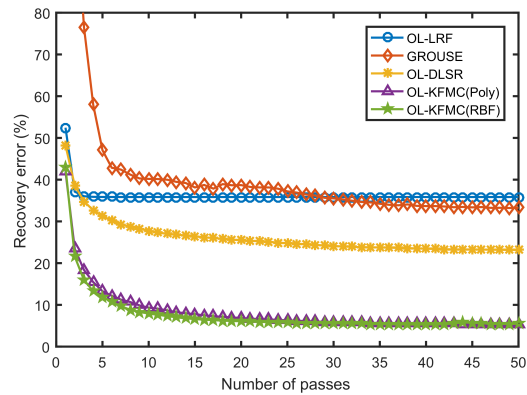}
\caption{Matrix completion errors of different passes}
\label{Fig_syn_on_pass}
\end{figure}

\subsection{Real data}
For the experiments of subspace clustering on incomplete data of Hopkins 155 datasets \cite{4269999}, similar to \cite{pmlr-v70-ongie17a}, we conducted the following procedures. First, the two subsets of video sequences, \textit{1R2RC} and \textit{1R2TCR}, were uniformly downsampled to 6 frames. Then the sizes of the resulted data matrices are $12\times 459$ and $12\times 556$. Second, we randomly removed a fraction ($10\%\sim 70\%$) of the entries of the two matrices and then perform matrix completion to recover the missing entries. Finally, SSC (sparse subspace clustering \cite{SSC_PAMIN_2013}) were performed to segment the data into different clusters. For fair comparison, the parameter $\lambda$ in SSC were chosen from $\lbrace 1,10,100\rbrace$ and the best results were reported.

For the CMU motion capture data, similar to \cite{NIPS2016_6357,pmlr-v70-ongie17a}, we use the trial $\#6$ of subject $\#56$ of the dataset, which is available at \textit{http://mocap.cs.cmu.edu/}. The data consists of the following motions: throw punches, grab, skip, yawn, stretch, leap, lift open window, walk, and jump/bound. The data size is $62\times 6784$. The data of each motion lie in a low-rank subspace and the whole data matrix is of full-rank \cite{NIPS2016_6357}. To reduce the computational cost and increase the recovery difficulty, we sub-sampled the data to $62\times 3392$. We considered two types of missing data pattern. The first one is randomly missing, for which we randomly removed $10\%$ to $70\%$ entries of the matrix. The second one is continuously missing, which is more practical and challenging. Specifically, for each row of the matrix, the missing entries were divided into 50 missing sequences, where the sequences are uniformly distributed and the length of each sequence is about $68\delta$. Here $\delta$ denotes the missing rate. The two missing data patterns are shown in Figure \ref{Fig_syn_on_pass}, in which the black pixel or region denotes the missing entries. For the online recovery, the number of passes for OL-LRF, GROUSE, OL-DLSR, and OL-KFMC are 10, 50, 10, and 5 respectively. The reason for this setting is that GROUSE requires large number of passes while the other methods especially our OL-KFMC requires fewer passes.

\begin{figure}[h!]
\centering
\includegraphics[width=7cm]{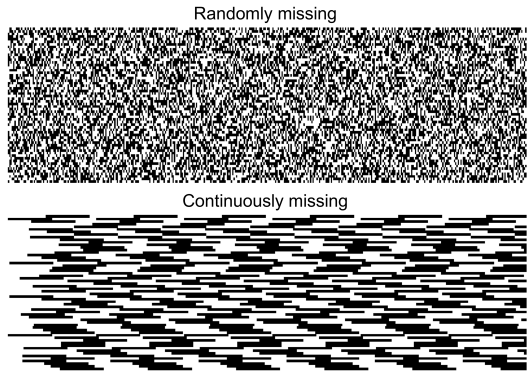}
\caption{Two missing data patterns for motion capture data}
\label{Fig_syn_on_pass}
\end{figure}

\end{document}